\documentclass[10pt]{article}
\usepackage[numbers]{natbib}
\usepackage{geometry}
\usepackage[utf8]{inputenc}
\usepackage[T1]{fontenc}
\usepackage{hyperref}
\usepackage{url}
\usepackage{booktabs}
\usepackage{amsfonts}
\usepackage{nicefrac}
\usepackage{microtype}

\usepackage{subcaption}
\usepackage{graphicx}
\usepackage{algorithm}
\usepackage{algorithmic}
\usepackage{textcomp}
\usepackage{xfrac}
\usepackage{xspace}

\newcommand{\eg}{e.g.\@\xspace}
\newcommand{\ie}{i.e.\@\xspace}
\usepackage{xcolor}
\usepackage{dirtytalk}
\usepackage{amsmath}
\usepackage{amsthm}

\usepackage[noabbrev]{cleveref}
\usepackage{tabularx}
\usepackage{wrapfig}
\usepackage{multirow}
\usepackage{caption}
\usepackage{soul}
\usepackage{enumitem}
\usepackage{float}
\usepackage{placeins}
\usepackage{mathtools}

\newtheorem{lemmat}{Lemma}
\DeclarePairedDelimiter\floor{\lfloor}{\rfloor}

\newcommand{\norm}[1]{\left\lVert#1\right\rVert}

\newcommand{\powersgd}{\textsc{PowerSGD}}
\newcommand{\topk}{\textsc{Top$K$}}
\newcommand{\adasp}{\textsc{Accordion}}
\newcommand{\cifart}{\textsc{Cifar-10}}
\newcommand{\cifarh}{\textsc{Cifar-100}}
\newcommand{\wikitext}{\textsc{Wikitext-2}}

\newcommand{\nccl}{\textsc{Nccl}}

\newcommand{\pytorch}{\textsc{PyTorch}}

\newcommand{\elltopkl}[1]{$\ell_\text{low} = \text{K}~#1\%$}
\newcommand{\elltopkh}[1]{$\ell_\text{high} = \text{K}~#1\%$}
\newcommand{\ellpsgdl}[1]{$\ell_\text{low} = \text{Rank}~#1$}
\newcommand{\ellpsgdh}[1]{$\ell_\text{high} = \text{Rank}~#1$}

\usepackage{bigstrut}
\newsavebox{\algleft}
\newsavebox{\algright}

\title{\adasp: Adaptive Gradient Communication via Critical Learning Regime Identification}
\author{\normalsize Saurabh Agarwal, Hongyi Wang, Kangwook Lee, Shivaram Venkataraman, Dimitris Papailiopoulos}
\date{University of Wisconsin-Madison}
\begin{document}
\maketitle







\vskip 0.3in

\begin{abstract}

Distributed model training suffers from communication bottlenecks due to frequent model updates transmitted across compute nodes.  To alleviate these bottlenecks, practitioners use gradient compression techniques like sparsification, quantization, or low-rank updates. The techniques usually require choosing a static compression ratio, often requiring users to 
balance the trade-off between model accuracy and per-iteration speedup.
In this work, we show that such performance degradation due to choosing a high compression ratio is not fundamental. An adaptive compression strategy can reduce communication while maintaining final test accuracy. Inspired by recent findings on critical learning regimes, in which small gradient errors can have irrecoverable impact on model performance, we propose \adasp\, a simple yet effective adaptive compression algorithm.
While \adasp\ maintains a high enough compression rate on average, it  avoids over-compressing gradients whenever in critical learning regimes, detected by a simple gradient-norm based criterion.
Our extensive experimental study over a number of machine learning tasks in distributed environments indicates that \adasp, maintains similar model accuracy to uncompressed training, yet achieves up to 5.5$\times$ better compression and up to 4.1$\times$ end-to-end speedup over static approaches. We show that \adasp\ also works for adjusting the batch size, another popular strategy for alleviating communication bottlenecks.

\end{abstract}



\section{Introduction}

Billion-parameter-scale neural networks and the rapid increase in compute requirements for training them has made distributed gradient-based methods a necessity. 
Synchronous, data-parallel training is the most widely adopted approach in this context, and requires combining the per-node gradient updates at every iteration~\cite{dean2012comm, iandola2016firecaffe, goyal2017accurate, shallue2018measuring}.
Communicating gradients frequently at such a large parameter scale leads to sub-optimal scalability in distributed implementations~\cite{dean2012comm, seide20141bit, alistarh2017qsgd, dawn2019scale, luo2020plink}.

To alleviate gradient communication bottlenecks, there are two main approaches proposed by prior work: (1) increasing the batch size~\cite{yao2018large, smith2017don, goyal2017accurate}, such that gradients are computed on a large batch by each worker thus reducing the frequency of per-epoch communication and (2) by performing lossy gradient compression~\cite{alistarh2017qsgd,vogels2019powersgd}, to reduce the size of the data communicated. Both of these methods involve navigating a trade-off between performance and accuracy.

It is a widely observed phenomenon that using large batch size can lead to degradation in final accuracy~\cite{yao2018hess, golmant2018computational, shallue2018measuring}. In response, several recent studies
 propose techniques to mitigate this accuracy loss, by using learning rate warmup~\cite{goyal2017accurate}, second order information~\cite{yao2018large, yao2018hess}, or layer-wise learning rate tuning~\cite{you2017scaling}. Some deployment challenges with these methods include
the need for significant amount of hyper-parameter tuning or running more epochs to converge to a good accuracy~\cite{golmant2018computational, shallue2018measuring}. 

On the other hand, when using gradient compression techniques like
low-precision training~\cite{alistarh2017qsgd, seide20141bit, wen2017terngrad, bernstein2018signsgd, acharya2019distributed}, \topk\ methods that exchange only the largest gradient coordinates~\cite{aji2017sparse, shi2019distributed, shi2019understanding, lin2017deep}, or methods that use low-rank based updates~\cite{wang2018atomo,vogels2019powersgd}, users need to specify an additional hyper-parameter that determines the degree of compression or sparsification before training begins. Choosing compression ratios presents a seemingly inherent trade-off between final model accuracy and the per-iteration communication overhead.
For instance, training ResNet-18 on \cifart\ using \topk\ with $K=10\%$ sparsification (i.e., where only 10\% of the top entries per gradient are communicated) takes around 3.6$\times$ less wall-clock time than training using $K=99\%$, but causes around 1.5\% degradation in final accuracy.

%

\begin{figure*}[t]
    \centering
    \includegraphics[width=0.9\textwidth]{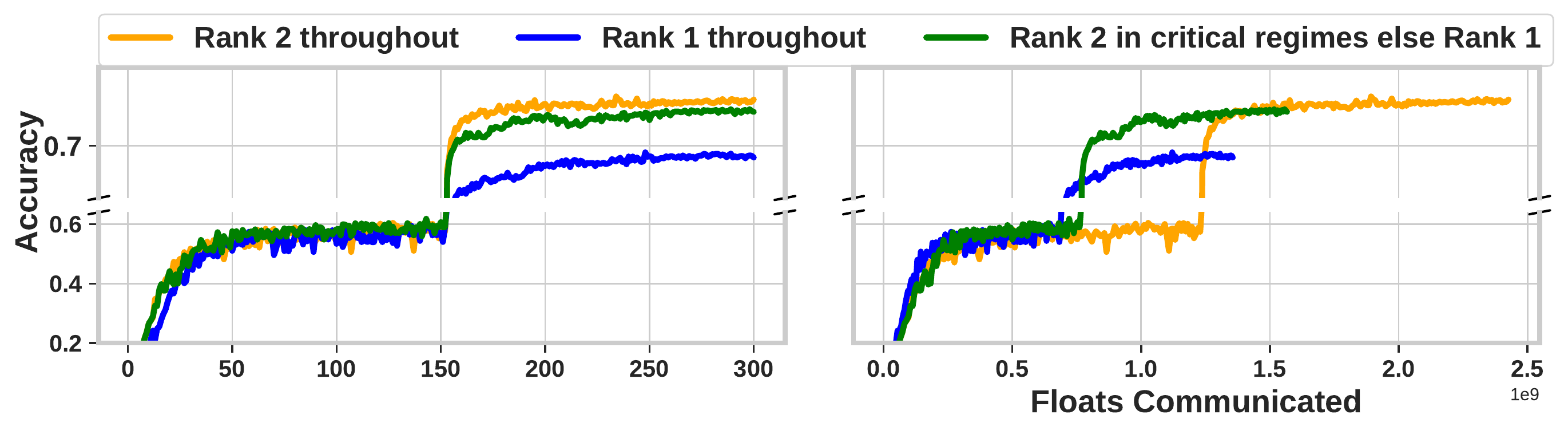}
    \caption{\small{\textbf{Effect of gradient compression in Critical Regimes when Training ResNet-18 on \cifarh:} We conduct these on a 4 Node cluster using \powersgd\ \textsc{Rank 2} (low compression) and \textsc{Rank 1} (high compression) (left) Accuracy vs Epochs, there exist a compression pattern which provides similar training accuracy as using low compression(\textsc{Rank 2}) throughout. (right) Accuracy vs Floats Communicated, this compression pattern communicates significantly less than using low compression throughout (\textsc{Rank 2}) }}
    \label{fig:critical_regimes_cifar100}
    \vspace{-0.1in}
\end{figure*}

%
%
%

%
%
This raises a fundamental question related to gradient compression:  \textit{Is this observed trade-off between communication and accuracy fundamental?}  
In this work, we first show that such a trade-off is \emph{not} fundamental but a mere artifact of using a \emph{fixed communication scheme} throughout training.
In other words, if a gradient communication schedule is chosen adaptively, then both the final model performance and communication efficiency can be improved, when compared against any fixed gradient communication schedule throughout training.
Figure~\ref{fig:critical_regimes_cifar100} shows an experiment where we train ResNet-18 on \cifarh, with different compression schemes.  We see that there exists an adaptive compression scheme that significantly reduces communication while maintaining final test accuracy.  

We attribute the power of adaptive schemes to the existence of \emph{critical regimes} in training. We build upon recent work by~\citet{achille2018critical} and ~\citet{jastrzebski2018on}, who show that even adding small noise to the data during critical regimes can result in poor generalization. We extend this notion to gradient communication and show that avoiding gradient compression during these critical regimes mitigates any loss in accuracy, as typically exhibited by non-adaptive communication-efficient training techniques. Thus, one can design an adaptive scheme that attempts to identify critical training regimes and adjusts the gradient compression rate accordingly.

Based on these findings, we propose \adasp, a simple but powerful gradient communication scheduling algorithm that is generic across models while imposing low computational overheads. \adasp\ inspects the change in the gradient norms to detect critical regimes and adjusts the communication schedule dynamically leading to performance improvements without sacrificing accuracy. We further show that \adasp\ works for both adjusting the gradient compression rate or the batch size without additional parameter tuning, hinting at a possible equivalence between the two.


Our experiments show that \adasp\ can achieve improved communication efficiency while maintaining high generalization performance on computer vision and language modelling tasks. 
In particular, using state-of-the-art (SOTA) sparsification (\topk) and low-rank approximation methods (\powersgd), \adasp\ achieves a reduction in communication of up to 3.7$\times$ and increase in speedup of up to 3.6$\times$, compared to using low compression throughout. At the same time \adasp\ achieves similar accuracy as the ``uncompressed'' vanilla SGD. When used for batch size tuning, we show that \adasp, \textit{without any hyper-parameter tuning}, is able to utilize extremely large batch sizes leading to reduction in communication of up to $5.5\times$ and speedup of $4.4\times$. We summarize our contributions below.

\paragraph{Contributions:}
\begin{itemize}[itemsep=-1pt,topsep=0pt]
    \item We show that gradient compression schemes need to be aware of critical regimes in training to mitigate any potential accuracy loss.
    \item We design \adasp\ an adaptive compression scheduling scheme that switches between low and high compression based on metrics that signal critical learning regimes. This allows for significant performance improvements without sacrificing accuracy.
    \item We provide extensive empirical evaluation of \adasp\ on several different neural networks using \powersgd\ and \topk\,, two state of the art gradient compression techniques, on several different data sets (\cifart, \cifarh, WikiText-2), showing \adasp\ reduces communication by $3.7\times$ without loss of accuracy or change in hyper-parameters.  
    \item We further show that \adasp\ can enable large batch training by switching between large and small batch sizes.  \adasp\, \textit{without any hyper-parameter tuning}, is able to reduce the time to accuracy by performing up to 5x fewer updates compared to using a small batch size.
\end{itemize}

\section{Related work}
\noindent\textbf{Lossy gradient compression.} Inspired by the fact that SGD can make good progress even with approximate gradients, various Gradient Compression methods have been proposed in recent literature. They can be broadly grouped into quantization, sparsification and low rank approximations. For quantization, \cite{seide20141bit, bernstein2018signsgd} replace each weight with just the sign values. While ~\cite{aji2017sparse, shi2019understanding, shi2019distributed, lin2017deep} use the largest few co-ordinates to create a sparse gradient. \citet{wangni2018gradient} randomly drop coordinates in the gradient update to create sparse gradient updates. For quantization \cite{alistarh2017qsgd, wen2017terngrad} quantize each gradient coordinate. In \cite{wang2018atomo,vogels2019powersgd} authors show that extremely low rank updates can achieve good compression without loss in accuracy. \citet{yu2018gradiveq} utilize correlation between gradients for linear compression. In \adasp\ our goal is to operate over an existing gradient compression method and provide reduction in communication without hurting generalization performance. 

\paragraph{Local SGD.} Unlike the lossy gradient compression methods which reduce the size of gradient updates, Local SGD methods  reduce the frequency of updates, by averaging weights every $\tau$ steps. \cite{Lin2020Don't, wang2020overlap, stich2019local, dutta2020slow} show that local SGD offers competitive performance on a variety of tasks. In this work we explicitly focus on reducing communication further using gradient compression or by varying batch size and plan to investigate if our insights can also apply for Local SGD in the future.

\paragraph{Adaptive communication.}  Wang et al.~\cite{wang2018adaptive} proposed an adaptive scheme that chooses the number of local steps $\tau$ adaptively, this method is applicable only on local SGD.
\citet{chen2018adacomp} proposed an auto-tuned compression method, but unlike \adasp\ it is a gradient compression method in itself and can't be applied with other methods. Recently \citet{guo2020acc} proposed an adaptive scheme to choose quantization factor for each coordinate in gradient vector, however in Figure~\ref{fig:icassp_compare} we observe that their method leads to some accuracy loss when used for \powersgd. 

\paragraph{Critical learning regimes.} In \cite{achille2018critical,Frankle2020The} authors highlighted the presence of critical regimes in neural network training. Various other works have highlighted the importance of early phases of training including, \citet{gur2018gradient} show gradient descent moves into a small sub-space after a few iterations, ~\cite{jastrzebski2018on, Jastrzebski2020The, keskar2016large} show that SGD is initially driven to difficult to optimize regimes. 
We leverage these insights to reduce communication when using gradient compression algorithms.

\paragraph{Batchsize scheduling.}
\cite{smith2017don, goyal2017accurate, hoffer2017train, you2017scaling, yin2017gradient}, show that large-batch SGD will hurt the eventual generalization performance. 
More surprisingly, \citet{you2017scaling} show that the use of large-batch SGD does not hurt the performance if used in later phases of the training. There are several works ~\cite{goyal2017accurate, you2017scaling, yao2018hess, smith2017don, devarakonda2017adabatch} which use adaptive batch size scaling. Either these methods require significant hyper-parameter tuning ~\cite{smith2017don} or require second order information ~\cite{yao2018large, yao2018hess}. ~\cite{yao2018large} does provide an adaptive method for batch size scheduling, but their method requires calculation of second order statistics which can often require more time than the gradient computation itself. In Sec.~\ref{sec:eval} we show that without any hyper-parameter tuning \adasp\ can enable large batch training and converge to same test accuracy with the same epoch budget as small batch training.
In Sec.~\ref{sec:batchsize_connection}, we show some non-trivial connection between \adasp\ and well-known suggestions on batch-size scheduling.

\section{Distributed SGD}\label{sec:prelim}
In this section, we formally describe the distributed SGD setting
Consider the standard synchronous distributed SGD setting with $N$ distributed workers~\cite{sergeev2018horovod}. 
For simplicity, we assume that each worker stores $n$ data points, giving us a total of $N\times n$ data points, say $\{(x_i, y_i)\}_{i=1}^{Nn}$. 

The goal is to find a model parameter $w$ that minimizes $f(w) = \frac{1}{Nn}\sum_{i=1}^{Nn} \ell(w;x_i,y_i)$ where $(x_i,y_i)$ is the $i$-th example. 
In particular, we minimize $f(w)$ using distributed SGD that operates as follows:
$w_{k+1} = w_{k}-\gamma_k \frac{1}{N} \sum_{i=1}^{N}\widehat{g}_i(w_{k})$ for $k \in \{0, 1, 2, \ldots\}$, where $w_0$ is the initial model, $\gamma_k$ is the step size, and $\widehat{g}_i(w)$ is a gradient computed at worker $i$ for a minibatch (of size $B$, with $B<n$).
\paragraph{Distributed SGD with adaptive gradient compression}
Vanilla distributed SGD incurs a huge communication cost per iteration that is proportional to the number of workers $N$ and the size of the gradient. 
To reduce this communication overhead, we consider a gradient compression strategy, say $C(\cdot, \ell)$, where $\ell$ is the parameter that determines the compression level used. 
With such a gradient compression strategy, the update equation becomes $w_{k+1} = w_{k}-\gamma_k \frac{1}{N} \sum_{i=1}^{N}C(\widehat{g}_i(w_{k}), \ell_k)$ for $k \in \{0, 1, 2, \ldots\}$, where communicating $C(\widehat{g}_i(w_{k}), \ell_k)$ requires much fewer bits than communicating the original gradients. 

\paragraph{Distributed SGD with adaptive batch size}
The number of communication rounds in a given epoch also depend on the batch size. For example a batch size $B_{high} > B_{low}$ will communicate $\floor*{\frac{B_{high}}{B_{low}}}$ times less than using batch size $B_{low}$ in a given epoch. Although the update equation remains the same $w_{k+1} = w_{k}-\gamma_k \frac{1}{N} \sum_{i=1}^{N}\widehat{g}_i(w_{k})$ for $k \in \{0, 1, 2, \ldots\}$, the number of steps $k$ taken by a model decreases by $\floor*{\frac{B_{high}}{B_{low}}}$ times for a fixed number of epochs. 

\paragraph{Goals} Our goal is to design an algorithm that automatically adapts the compression rate $\{\ell_k\}$  or batch size $B_k$ while training. Although the interplay between batch size and compression ratio is interesting, we don't explore these together, \ie we don't vary batch size when training with gradient compression. Here, we consider a centralized algorithm, \textit{i.e.}, one of the participating nodes decides $\ell_{k+1}$ or $B_{k+1}$ based on all the information available up to times $k$.
This communication rate is then shared with all the $N$ workers so that they can adapt either their compression ratio or batch size. 

\section{\adasp}
\begin{figure*}[t]
    \begin{center}
    \begin{subfigure}[b]{0.33\textwidth}
    \includegraphics[width=\textwidth]{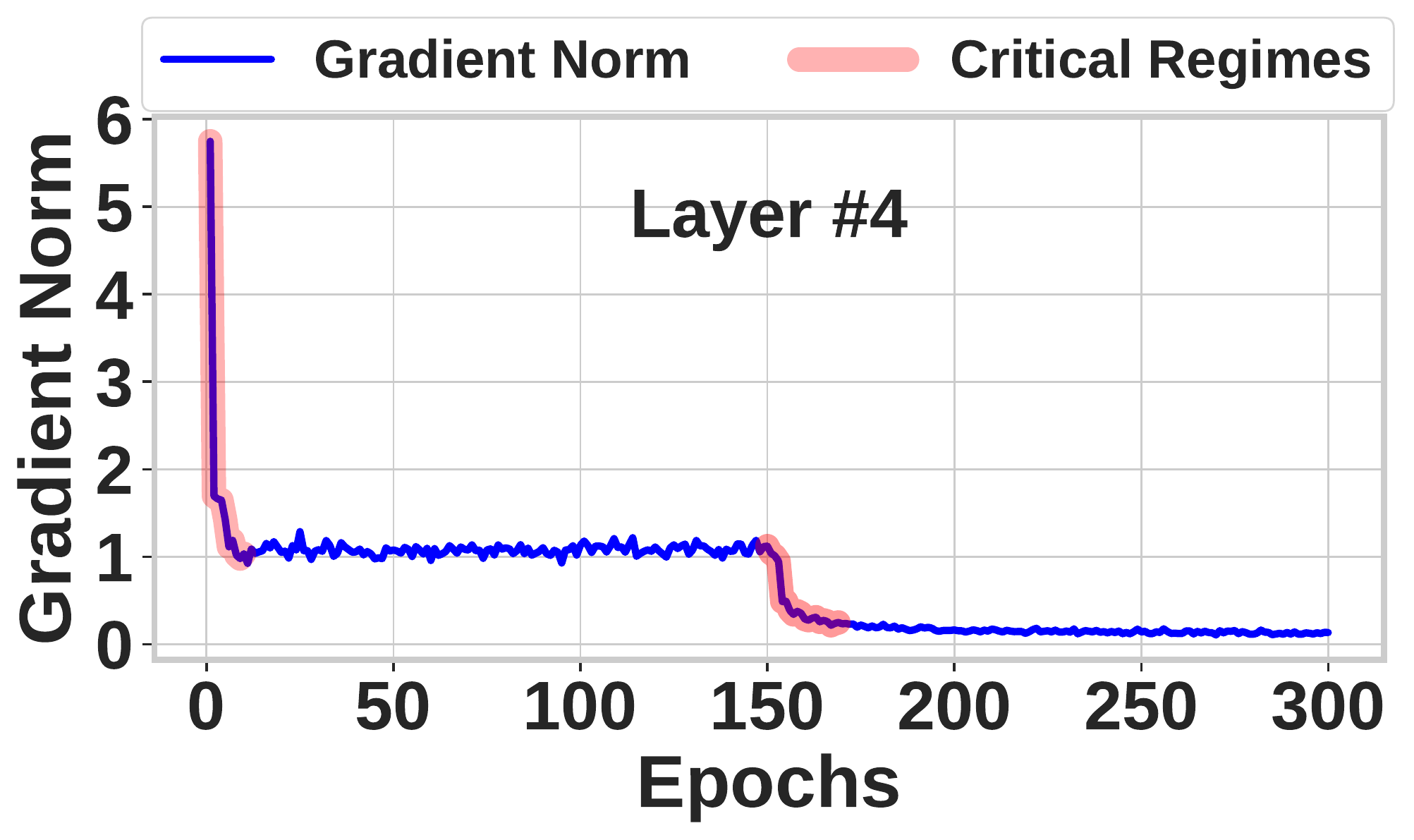}
    \caption{Critical Regimes}
     \label{fig:crit_res18_cifar100}
    \end{subfigure}
    \begin{subfigure}[b]{0.66\textwidth}
    \includegraphics[width=\textwidth]{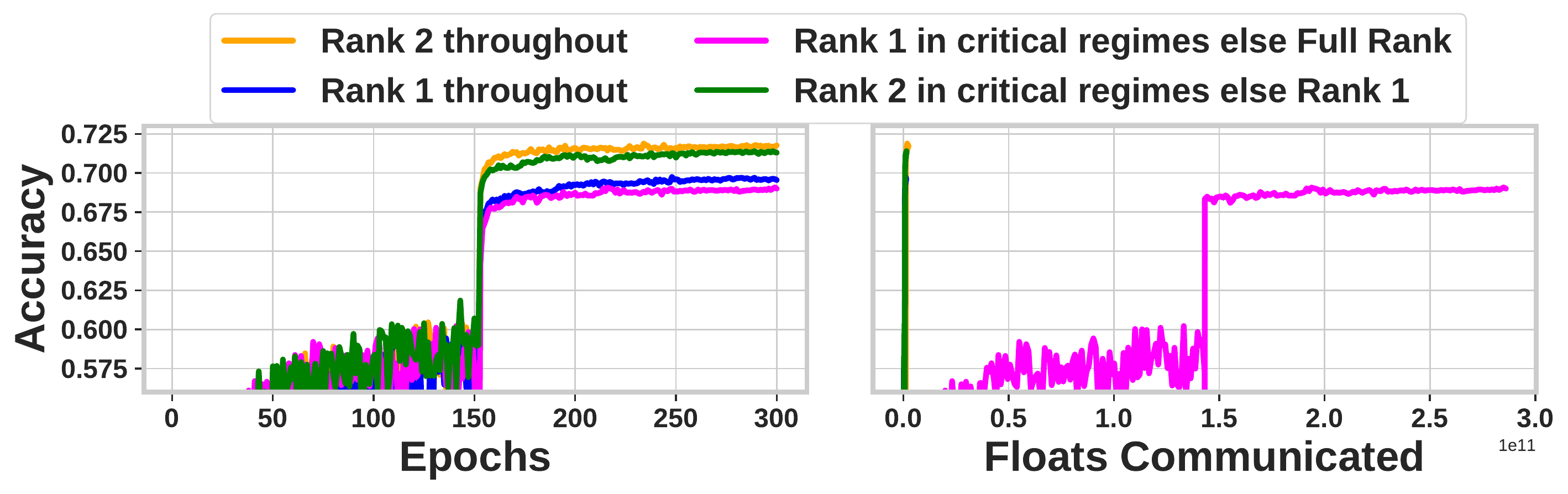}
    \caption{Accuracy vs Epochs and Floats Communicated}
    \label{fig:critical_regimes_full_rank}
    \end{subfigure}
    \end{center}
    \caption{\small{\textbf{Effect of gradient compression in Critical Regimes when Training ResNet-18 on \cifarh:} (a) Critical regimes in \cifarh, ResNet-18 (b, Left) Accuracy vs Epochs. Show the significance of critical regimes in training, using low compression(Rank 2) in critical regimes is enough to get similar accuracy as using low compression throughout . (b, Right) Accuracy vs Floats Communicated, Even when we use uncompressed (Full Rank)  gradients everywhere but use high compression (Rank 1) in critical regimes  it is not possible to bridge the accuracy gap.}}
    \vspace{-0.2in}
\end{figure*}

In this section, we first explain why adaptive gradient communication can help maintain high generalization performance while minimizing the communication cost. 
We study this first with gradient compression techniques and then based on these insights we propose \adasp\ a gradient communication scheduling algorithm. 
Finally, we show that there is a connection between batch size and gradient compression, and thus \adasp\ can also be used to enable large batch training without accuracy loss.  

\subsection{Adaptive communication using critical regimes}
\label{sec:critical_connection}

Recent work by \citet{achille2018critical} has identified \emph{critical regimes}  or phases of training that are important for training a high quality model. In particular, \citet{achille2018critical} show that the early phase of training is critical. They perform an experiment where the first few epochs have corrupted training data and then continue training the DNN with clean training data for the rest of the epochs.
Surprisingly, the DNN trained this way showed a significantly impaired generalization performance no matter how long it was trained with the clean data after the critical regime. 

We extend these ideas to aid in the design of an adaptive communication schedule and first study this using \powersgd\ as the gradient compression scheme. We begin by observing how the gradient norm for each layer behaves while training.
When training ResNet-18 on \cifarh, in Figure~\ref{fig:crit_res18_cifar100} we see two regions where gradient norm decreases rapidly; during the first $20$ epochs and the $10$ epochs right after the $150$-th epoch, \ie{}, the point at which learning rate decay occurs. We experimentally verify that these two regions are critical by considering the following compression schedule $\ell = \textsc{low}$ for the first $20$ epochs and for $10$ epochs after the $150$ epoch, and $\ell = \textsc{high}$ elsewhere. Under this scheme the gradients will not be over-compressed in the critical regimes, but at the same time the overall communication will be close to high compression. Figure~\ref{fig:critical_regimes_full_rank} shows the experimental results with ResNet-18 on \cifarh\ for the above scheme. It can be observed that just using low compression (rank 2) in these critical regimes and high compression (rank 1) elsewhere is sufficient to get the same accuracy as using low compression throughout while reducing communication significantly.

Interestingly we also observe in Figure~\ref{fig:critical_regimes_full_rank} that any loss in accuracy by using high compression in critical regimes is not recoverable by using low compression elsewhere. For instance, consider the following compression schedule: $\ell = \textsc{high compression rate}$ for first $20$ epochs and for $10$ epochs after the $150$ epoch, and $\ell = \textsc{no compression}$ elsewhere.
Under this schedule, gradients will be over-compressed in the critical regimes, but will be \emph{uncompressed} elsewhere. We see that for ResNet-18 on \cifarh\ even with significantly higher communication one can not overcome the damage done to training by over compressing in critical regimes.
We hypothesize that in critical regimes, SGD is navigating to the steeper parts of the loss surface and if we use over-compressed gradients in these regimes, then the training algorithm might take a different trajectory than what SGD would have taken originally. This might cause training to reach a sub-optimal minima leading to degradation in final test accuracy. 




\begin{figure*}[t]
    \begin{center}
    
    \begin{subfigure}[b]{0.49\textwidth}
        \includegraphics[width=\textwidth]{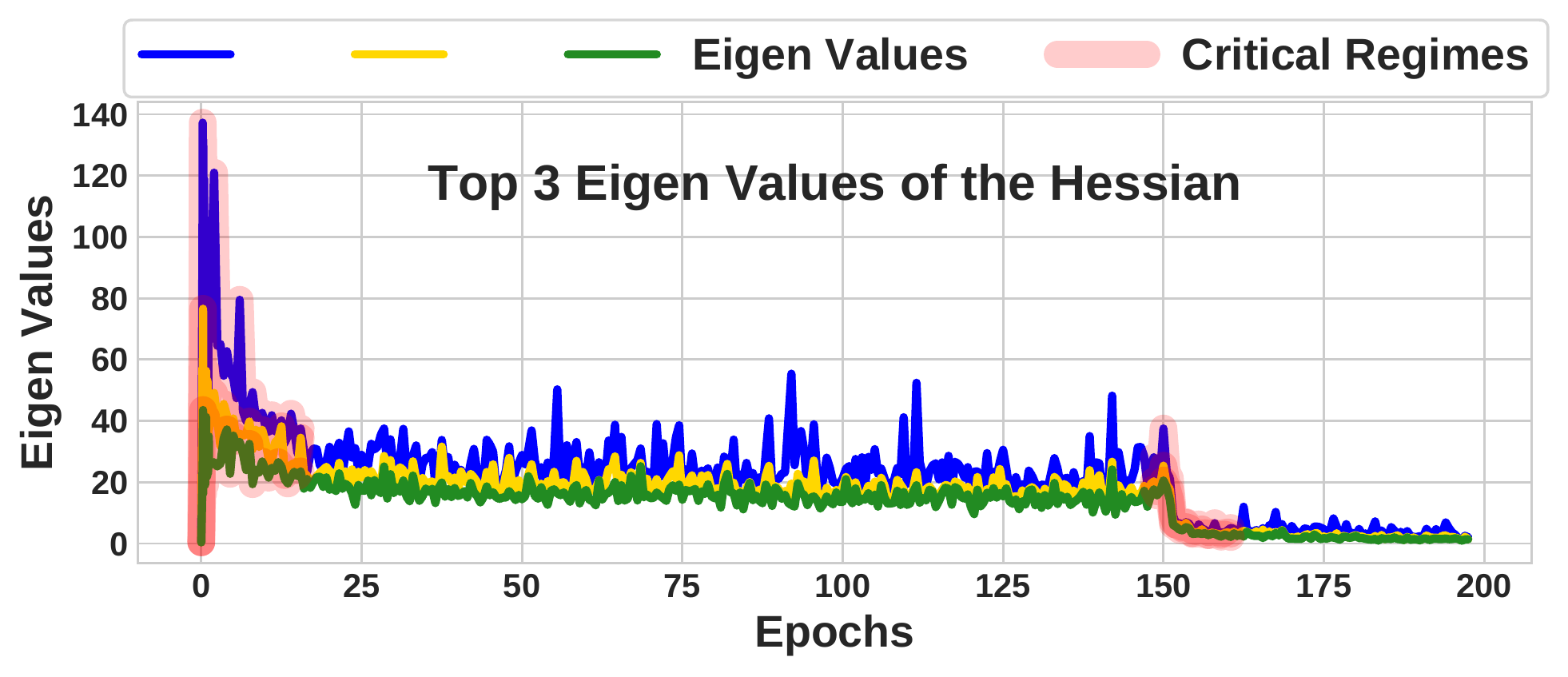}
        \caption{Critical Regimes based on Hessian}
        \vspace{-0.1in}
        \label{fig:critical_hessian}
    \end{subfigure}
    \begin{subfigure}[b]{0.49\textwidth}
        \includegraphics[width=\textwidth]{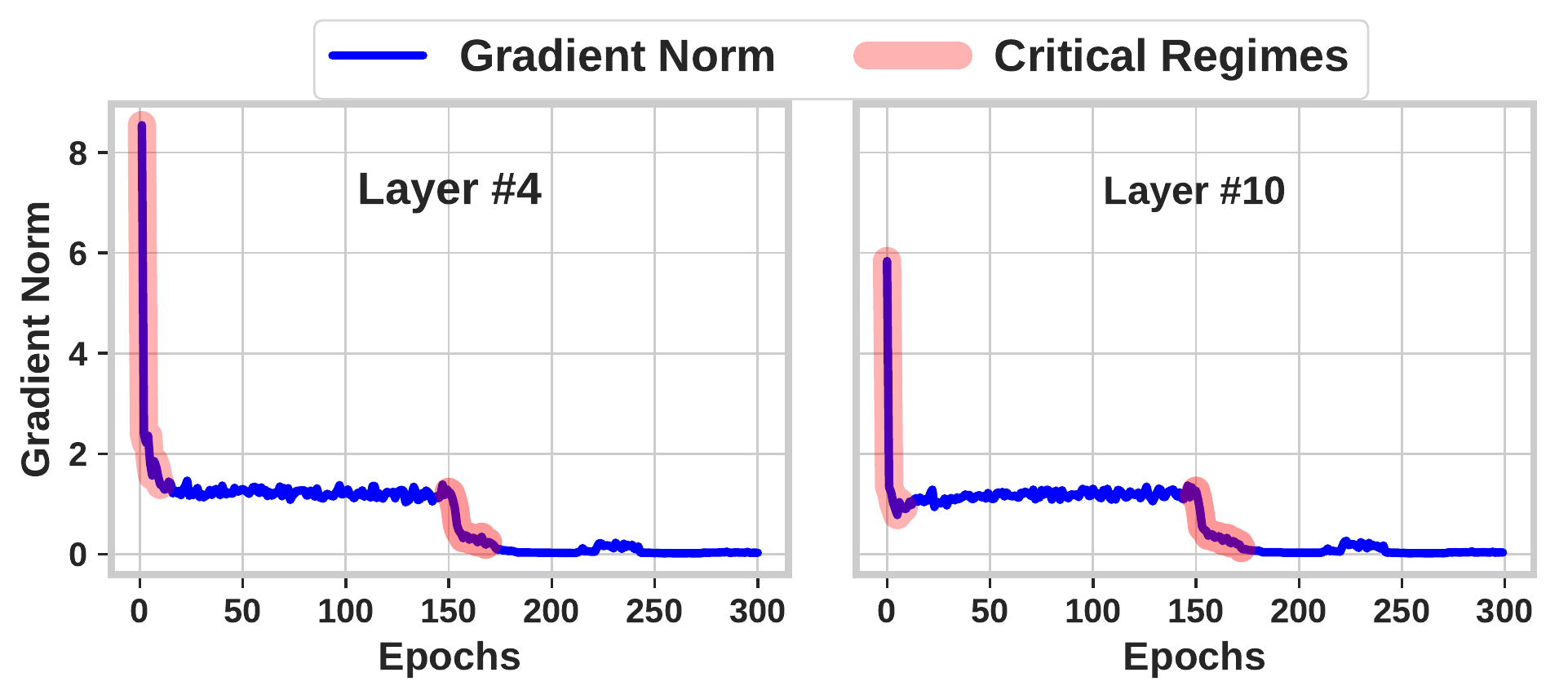}
        \caption{Critical Regimes based of Gradient Norm}
        \vspace{-0.1in}
        \label{fig:critical_grad}
    \end{subfigure}
    \end{center}
    \caption{\small{\textbf{Comparison of Critical Regimes found using Analysis of eigenvalues of Hessian vs Using the Norm of the Gradient:}} The experiment is performed on ResNet-18, for \cifart. We show that Critical Regimes detected by rapid decay in top eigenvalues of Hessian can also be detected using decay in gradient norm}
    \label{fig:critical_compare}
    \vspace{-0.2in}
\end{figure*}
    

\paragraph{Detecting Critical Regimes:}
Prior work for detecting critical regimes~\cite{jastrzebski2018on} used the change in eigenvalues of the Hessian as an indicator.
We next compare the critical regimes identified by the gradient norm approach described above with the approach used in \citet{jastrzebski2018on}.
In Figure~\ref{fig:critical_compare}, we show that these two approaches yield similar results for ResNet-18 on \cifart, with the latter having an advantage of being orders of magnitude faster to compute.

Thus, we can see that finding an effective communication schedule is akin to finding critical regimes in neural network training and these \textit{critical regimes} can be identified by measuring the change in gradient norm. 

\subsection{\adasp's Design}
We now provide a description of \adasp, our proposed algorithm that automatically switches between lower and higher communication levels by detecting critical regimes.
\adasp's first goal is to identify critical regimes efficiently.
Our experiments, as discussed previously (Figure~\ref{fig:critical_compare}), reveal that critical regimes can be identified by detecting the rate of change in gradient norms without using the computationally expensive technique of \cite{jastrzebski2018on, Jastrzebski2020The, keskar2016large}, where eigenvalues of the Hessian are used to detect critical regimes.  
This leads us to propose the following simple way to detect critical regimes:
\[
\frac{\big\lvert\norm{\Delta_{\text{old}}} - \norm{\Delta_\text{curr}} \big \rvert}{\norm{\Delta_\text{old}}} \geq \eta,
\]
where $\Delta_\text{curr}$ and  $\Delta_\text{prev}$, denotes the accumulated gradient in the current epoch and some previous epoch respectively, and $\eta$ is the threshold used to declare critical regimes.
We set $\eta = 0.5$ in all of our experiments.

We show the \adasp\ algorithm for gradient compression in Algorithm~\ref{alg:adasparse}.
For simplicity and usability, \adasp\ only switches between two levels of compression levels: $\ell_\text{low}$  and $\ell_\text{high}$.
Once \adasp\ detects critical regimes, it sets the compression level as $\ell_\text{low}$ to avoid an undesirable drop in accuracy.
Based on our observation, critical regimes also almost always occur after learning rate decay, therefore we let \adasp\ declare critical regime after every learning rate decay.
If \adasp\ detects that the critical phase ends, it changes the compression level to $\ell_\text{high}$ to save communication cost. For batch size we use the same algorithm, except instead of switching between $\ell_\text{low}$ and  $\ell_\text{high}$ we switch between $B_\text{low}$ and $B_\text{high}$. 

We remark that \adasp\ operates at the granularity of the gradient compressor being used.
For instance, \powersgd\ approximates the gradients of each layer independently, so \adasp\ will also operate at each layer independently and provide a suitable compression ratio for each layer in an adaptive manner during training. While batch size scheduling operates at the whole model so \adasp\ looks at the gradient of whole model and chooses a suitable batch size.

\paragraph{Computational and memory overhead:}
\adasp\ accumulates gradients of each layer during the backward pass. After each epoch, norms are calculated, creating $\norm{\nabla_{curr}}$. Once the compression ratio is chosen $\norm{\nabla_{curr}}$ becomes $\norm{\nabla_{old}}$. Thus requiring only size of the model(47 MB in ResNet-18) and a few float values worth of storage. Also \adasp\ only uses the ratio between previous and current gradient norms to detect critical regimes. This allows \adasp{} to be easily integrated in a training job where gradients are already calculated, thus making the computational overhead negligible.

\begin{algorithm}[tb]
\caption{\adasp\ for Gradient Compression }
\label{alg:adasparse}
\begin{algorithmic}
    \STATE {\bfseries HyperParameters:}compression levels \{$\ell_\text{low}$, $\ell_\text{high}$\} and detection threshold $\eta$
    \STATE {\bfseries Input:} {accumulated gradients in the current epoch ($\Delta_\text{curr}$) and in the previous epoch ($\Delta_\text{prev}$)}
    \STATE {\bfseries Input:} {learning rate of the current epoch ($\gamma_\text{curr}$) and of the next epoch ($\gamma_\text{next}$)}
    \STATE {\bfseries Output:} {compression ratio to use $\ell$}
    \IF{$\sfrac{\lvert\norm{\Delta_\text{prev}} - \norm{\Delta_\text{curr}}\rvert}{\norm{\Delta_\text{prev}}} \geq \eta$ or $\gamma_{next} < \gamma_{curr}$} 
        \STATE \textbf{return} $\ell_\text{low}$
    %
    \ELSE
        \STATE \textbf{return} $\ell_\text{high}$
    \ENDIF
 \end{algorithmic}
\end{algorithm}

\subsection{Relationship between gradient compression and adaptive batch-size}
\label{sec:batchsize_connection}
 
\begin{figure*}[ht]
    \begin{center}
    \begin{subfigure}[b]{0.33\textwidth}
        \includegraphics[width=\textwidth]{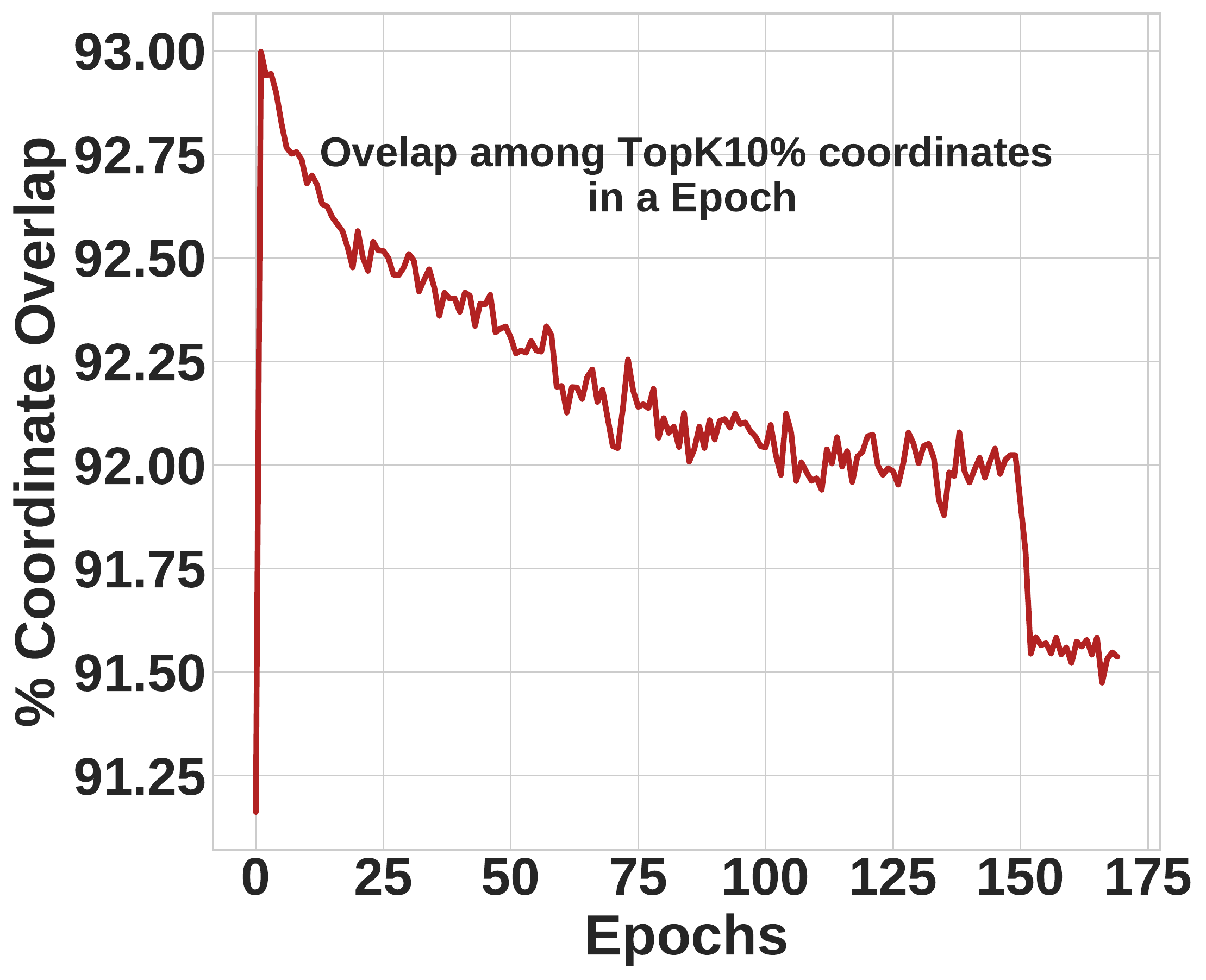}
        \caption{Overlap in coordinates }
        \vspace{-0.1in}
        \label{fig:topk_overlap}
    \end{subfigure}
    \begin{subfigure}[b]{0.66\textwidth}
        \includegraphics[width=\textwidth]{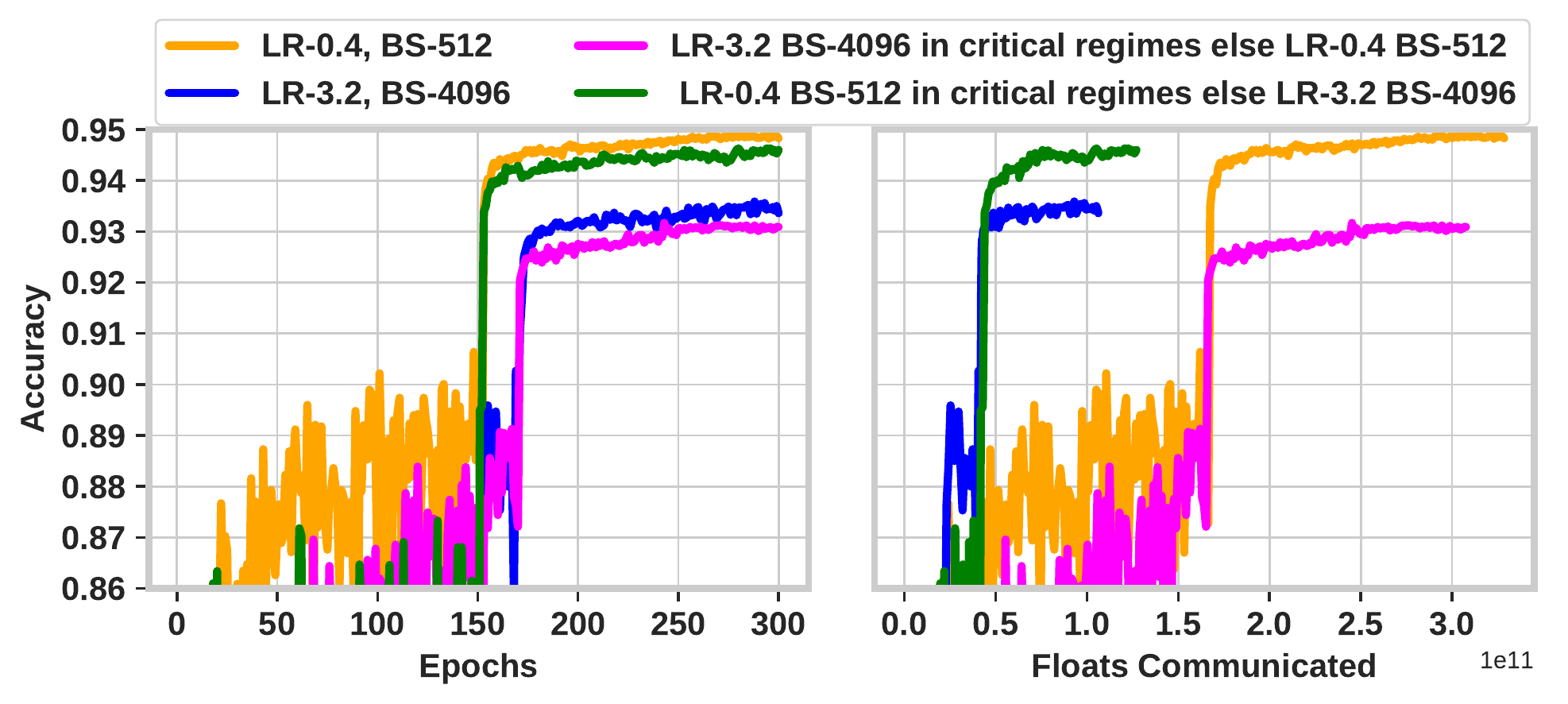}
        \caption{Effect of Different Batch sizes in critical regimes}
        \vspace{-0.1in}
        \label{fig:critical_batchsize}
    \end{subfigure}
    \end{center}
    \caption{\small{\textbf{Effect of batch size (ResNet-18 on \cifart):} (a) We show that there is significant overlap among the \textsc{Top}10\% coordinates. (b, left) Shows that using small batches only in critical regimes is enough to get performance similar to using small batches everywhere. We scale learning rate linearly with batch size as in~\cite{goyal2017accurate}, at steps 150 and 250 we decay the learning rate by 10 and 100 respectively. (b, right) accuracy vs communication.}}
    
    \label{fig:critical_regimes}
    \vspace{-0.1in}
\end{figure*}


We first evaluate the effect of batch size on neural network training through the lens of \emph{critical regimes}, which suggests using small batch sizes in critical regimes and large batch size outside critical regimes should not hurt test accuracy. We empirically show in Figure~\ref{fig:critical_batchsize} that this is indeed true. 

Next, the connection between compression and batch size tuning can be made more formal under the following assumption: ``each stochastic gradient is the sum of a sparse mean and a dense noise'', i.e.,
\begin{equation}
\begin{split}
\nabla_w \ell(w;x_i,y_i)& = \underbrace{\mathbb{E}_j \nabla_w \ell(w;x_j,y_j)}_{\text{sparse, large magnitudes}}\\  &+\underbrace{(\nabla_w \ell(w;x_i,y_i)-\mathbb{E}_j \nabla_w \ell(w;x_j,y_j))}_{\text{dense, small magnitudes}}
\end{split}
\end{equation}

Under this assumption, we can see that ``large batch gradient $\approx$ highly compressed gradient'', as a large batch gradient will be close to $\mathbb{E}_j \nabla_w \ell(w;x_j,y_j)$ by the law of large numbers, a highly compressed gradient will also pick up the same sparse components. 
Similarly, a small batch gradient is equivalent to a weakly compressed gradient. We would like to point out  that this assumption is not general and is not applicable on all data or models. It will only hold for models trained with sparsity inducing norms. 

We also conduct a simple experiment to support our intuition. 
We collect all stochastic gradients in an epoch and compute the overlap in coordinates of \textit{Top10\%} entries to find how much their supports overlap.
Figure~\ref{fig:topk_overlap} shows that $>90\%$ of the top-$K$ entries are common between a pair of stochastic gradients, thereby justifying the above gradient modeling. 


Thus, our findings along with prior work in literature can be summarized as high gradient compression, noisy training data, or large batch size in the critical regimes of training hurts generalization. We study this connection further in Appendix~\ref{sec:gradcompreesion_connection_batch_size}. 
This connection also suggests that \adasp\ can also be used to schedule batch size and in Section~\ref{sec:eval} we evaluate this.
\section{Experimental evaluation}
\label{sec:eval}
We experimentally verify the performance of \adasp\ when paired with two SOTA gradient compressors, \textit{i.e.}, (i) \powersgd~\cite{vogels2019powersgd}, which performs low-rank gradient factorization via a computationally efficient approach, and (ii) \topk\ sparsification~\cite{aji2017sparse}, which sparsifies the gradients by choosing the $K$ entries with largest absolute values. Further we also use \adasp\ to schedule batch size switching between batch size 512 and 4096 for \cifarh and \cifart . 

\subsection{Experimental setup}

\paragraph{Implementation:}
We implement \adasp\ in PyTorch~\cite{paszke2019pytorch}. All experiments were conducted on a cluster that consists of $4$ \texttt{p3.2xlarge} instances on Amazon EC2. Our implementation used \nccl{}, an optimized communication library for use with NVIDIA GPUs. For \powersgd\ and Batch Size experiments we used the all-reduce collective in NCCL and for \topk\ we used the all-gather collective.

\paragraph{Hyperparameters:}
\label{sec:hyper_param}
We fix $\eta$ to be 0.5 and run \adasp\ every 10 epochs \ie \adasp\ detects critical regimes by calculating rate of change between gradients accumulated in current epoch and the gradients accumulated 10 epochs back. We empirically observe that these choices of hyper-parameters lead to good results and have not tuned them. One of our primary goals was to design \adasp\ such that it should not require a significant amount of hyper-parameter tuning. Therefore for all of our experiments  we didn't perform any hyper-parameter tuning and used the same hyper-parameters as suggested by authors of previous compression methods, \eg For \powersgd\ we used the same setting as suggested by \citet{vogels2019powersgd}. For large batch size experiments we use the same hyper-parameters as used for regular training. For all our experiments on batch size we performed LR Warmup of 5 epochs as suggested by ~\citet{goyal2017accurate}, \ie for batch size 512 we linearly increase the learning rate from 0.1 to 0.4 in five epochs where 0.1 is learning rate for batch size 128. Due to relationship shown between batch Size and learning rate by~\citet{smith2017don,devarakonda2017adabatch} when  \adasp\ shifts to large batch it also correspondingly increases the learning in the same ratio, \ie when switching between Batch Size 512 to Batch size 4096, \adasp\ also scales the learning rate by $8\times$. Detailed experimental setup can be found in the Appendix. 


\paragraph{Dataset and Models:}
For image classification tasks we evaluated \adasp\ on \cifart\ and \cifarh~\cite{krizhevsky2009learning}. \cifart\ consists of 50,000 train images and 10,000 test images for 10 classes. \cifarh\ has the same number of samples but for 100 classes. For language modeling we used \wikitext\ which has around 2 million train tokens and around 245k testing tokens. 
We used standard prepossessing steps, details of which can be found in the Appendix.
To show the wide applicability of \adasp\ we consider a number of model architectures. For CNNs, we study networks both with and without skip connections. VGG-19~\cite{simonyan2014very} and GoogleNet~\cite{szegedy2015going} are two networks without skip connections. While ResNet-18~\cite{he2016deep}, Densenet~\cite{huang2017densely}, and Squeeze-and -Excitation~\cite{hu2018squeeze} are networks with skip connections.  For language tasks we used a two layer LSTM. More details about the models can be found in Appendix.


\paragraph{Metrics:}
We evaluate \adasp\ against high communication training on three different metrics: (i) accuracy; (ii) communication savings; (iii) total wall clock time saved. We train each method for the same number of epochs with the hyper-parameters suggested in prior literature~\cite{goyal2017accurate, vogels2019powersgd, aji2017sparse, smith2017don}. We report the mean test accuracy reached after three independent trials. Our error bars report $95\%$ confidence intervals.

\subsection{Results}
\label{sec:results}
\adasp's performance is summarized in Tables~\ref{tab:cifar10psgd} to~\ref{tab:cifar100batch}. For each model we state the accuracy achieved when using low communication, high communication, and compare it to using \adasp\ to automatically switch between low and high communication. Detailed convergence curves with error bars can be found in Appendix. 
\subsection{\adasp\ with \powersgd}
 \begin{table*}[ht]
\begin{minipage}{0.485\textwidth}
\centering
\caption{ \small{\adasp\ with \powersgd\ on \cifart\ }}
\label{tab:cifar10psgd}
\resizebox{\linewidth}{!}{%
\begin{tabular}{@{}lllllll@{}}
\toprule
\textbf{Network} &
  \textbf{Rank} &
  \textbf{Accuracy} &
  \multicolumn{2}{l}{\textbf{\begin{tabular}[c]{@{}l@{}}Data Sent\\(Million Floats)\end{tabular}}} &
  \multicolumn{2}{l}{\textbf{\begin{tabular}[c]{@{}l@{}}Time\\ (Seconds)\end{tabular}}} \\ \midrule
\multirow{3}{*}{Resnet-18} & Rank 2    & \textbf{94.5\%} & 2418.4 & {($1 \times$)} & {3509} & {($1 \times$)} \\ \cmidrule(l){2-7} 
                           & Rank 1    & 94.1\% & 1350.4 & ($1.7 \times$) & 3386 & {($1.03 \times$)}\\ \cmidrule(l){2-7} 
                           & \adasp & \textbf{94.5\%} & 1571.8 & {(\textbf{1.5}$\mathbf{\times}$)} & 3398 & {(\textbf{1.03}$\mathbf{\times}$)} \\ \midrule
\multirow{3}{*}{VGG-19bn}  & Rank 4    & \textbf{93.4\%} & 6752.0 & {($1 \times$)} & 3613 & {($1 \times$)} \\ \cmidrule(l){2-7} 
                           & Rank 1    & 68.6\% & 2074.9 & {($3.25 \times$)} & 3158 & {($1.14 \times$)} \\ \cmidrule(l){2-7} 
                           & \adasp & \textbf{92.9\%} & 2945.1 & {(\textbf{2.3}$\mathbf{\times}$)} & 3220 & {(\textbf{1.12}$\mathbf{\times}$)} \\ \midrule
\multirow{3}{*}{Senet}     & Rank 4    & \textbf{94.5\%}& 4361.3 & {($1 \times$)} & 4689 & {($1 \times$)} \\ \cmidrule(l){2-7} 
                           & Rank 1    & 94.2\% & 1392.6 & {($3.1 \times$)} & 4134 & {($1.13 \times$)} \\ \cmidrule(l){2-7} 
                           & \adasp & \textbf{94.5\%} & 2264.4 & {(\textbf{1.9}$\mathbf{\times}$)} & 4298 & {(\textbf{1.09}$\mathbf{\times}$)} \\ \bottomrule
\end{tabular}}
\end{minipage}\quad
\begin{minipage}{0.49\textwidth}
\centering
\caption{\small{\adasp\ with\hspace{0.01cm}\powersgd\ on \cifarh }}
\label{tab:cifar100psgd}
\resizebox{\linewidth}{!}{%
\begin{tabular}{@{}lllllll@{}}
\toprule
\textbf{Network} &
  \textbf{Rank} &
  \textbf{Accuracy} &
  \multicolumn{2}{l}{\textbf{\begin{tabular}[c]{@{}l@{}}Data Sent\\(Million Floats)\end{tabular}}} &
  \multicolumn{2}{l}{\textbf{\begin{tabular}[c]{@{}l@{}}Time\\ (Seconds)\end{tabular}}} \\ \midrule
\multirow{3}{*}{Resnet-18} & Rank 2    & \textbf{71.7}\% & 2426.3 & {($1 \times$)} & 3521 & {($1 \times$)} \\ \cmidrule(l){2-7} 
                           & Rank 1    & 70.0\% & 1355.7 & {($1.8 \times$)} & 3388 & {($1.04 \times$)}\\ \cmidrule(l){2-7} 
                           & \adasp & \textbf{71.8}\% & 1566.3 & {(\textbf{1.6}$\mathbf{\times}$)} & 3419 & {(\textbf{1.03}$\mathbf{\times}$)} \\ \midrule
\multirow{3}{*}{DenseNet}  & Rank 2    & \textbf{72.0\%} & 3387.4 & {($1 \times$)} & 13613 & {($1 \times$)} \\ \cmidrule(l){2-7} 
                           & Rank 1    & 71.6\% & 2155.6 & {($1.6 \times$)} & 12977 & {($1.04 \times$)} \\ \cmidrule(l){2-7} 
                           & \adasp & \textbf{72.5\%} & 2284.9 & {(\textbf{1.5}$\mathbf{\times}$)} & 13173 & {(\textbf{1.03}$\mathbf{\times}$)} \\ \midrule
\multirow{3}{*}{Senet}     & Rank 2    & \textbf{72.5\%} & 2878.1 & {($1 \times$)} & 5217 & {($1 \times$)} \\ \cmidrule(l){2-7} 
                           & Rank 1    & 71.5\% & 1683.1 & {($1.7 \times$)} & 4994 & {($1.04 \times$)} \\ \cmidrule(l){2-7} 
                           & \adasp & \textbf{72.4\%} & 2175.6 & {(\textbf{1.3}$\mathbf{\times}$)} & 5074 & {(\textbf{1.03}$\times$)} \\ \bottomrule
                           
\end{tabular}}
\end{minipage}
\vspace{-0.1in}
\end{table*}
\powersgd~\cite{vogels2019powersgd} shows that using extremely low rank updates (Rank-2 or Rank-4) with error-feedback~\cite{stich2019error} can lead to the same accuracy as syncSGD. In Table~\ref{tab:cifar10psgd} and~\ref{tab:cifar100psgd}  we show that \adasp\ by performing adaptive switching between Rank-1 and Rank-2,4 reaches similar accuracy but with significantly less communication. For \eg in~\Cref{tab:cifar100psgd} with ResNet-18 on \cifarh\, using $\ell_\text{low}=\text{Rank}~2$ leads to accuracy of $72.4\%$ while $\ell_\text{high}=~\text{Rank}~1$ achieves $71.3\%$. \adasp\ switching between \textsc{Rank~2} and \textsc{Rank~1} achieves an accuracy of $72.3\%$. Figure~\ref{fig:powersgd_vgg19bn} shows the result for VGG-19bn trained with \cifart, in this case \adasp\ almost bridges the accuracy gap of 25\% while saving almost $2.3\times$ in communication.
\begin{figure*}[t]
    \begin{center}
         \includegraphics[width=0.9\textwidth]{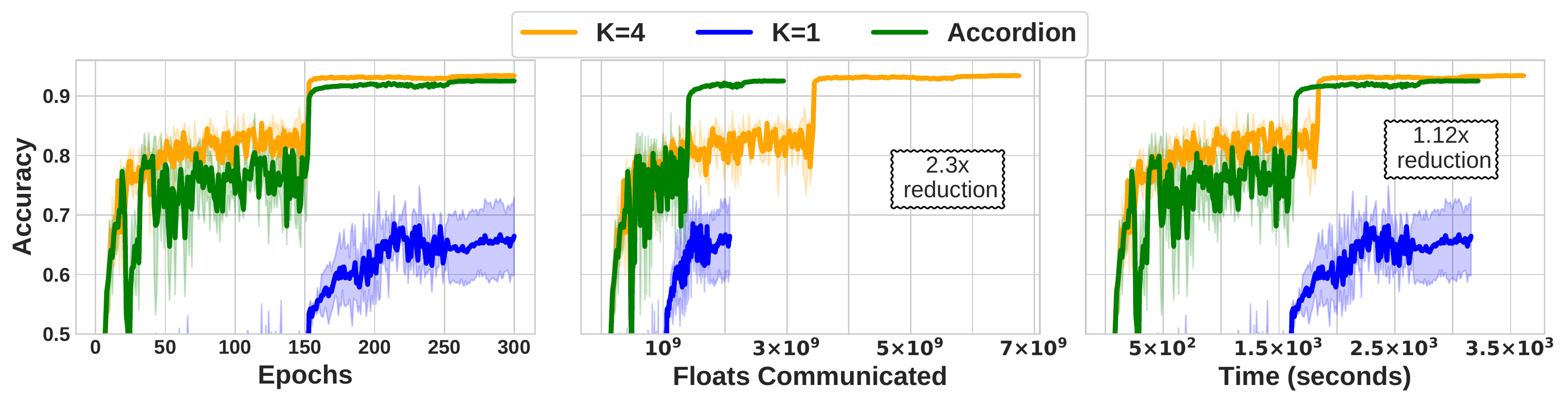}
    \end{center}
    \caption{\small{\textbf{\adasp\ using \powersgd\ with $\ell_\text{low} = \text{rank}~4$ and $\ell_\text{high} = \text{rank}~1$ on VGG-19bn:} We show \adasp\ being able to bridge more than 25\% of accuracy difference with 2.3$\times$ less communication .
    }}
    \label{fig:powersgd_vgg19bn}
\end{figure*}

\subsection{\adasp\ with \topk}
\begin{table*}[ht]
\begin{minipage}{0.485\textwidth}
\caption{\small{\adasp\ using \topk\ on \cifart\ } }
\label{tab:cifar10topk}
\resizebox{\linewidth}{!}{
\begin{tabular}{@{}lllllll@{}}
\toprule
\textbf{Network} &
  \textbf{K(\%)} &
  \textbf{Accuracy} &
  \multicolumn{2}{l}{\textbf{\begin{tabular}[c]{@{}l@{}}Data Sent\\(Billion Floats)\end{tabular}}} &
  \multicolumn{2}{l}{\textbf{\begin{tabular}[c]{@{}l@{}}Time\\ (Seconds)\end{tabular}}} \\ \midrule
\multirow{3}{*}{Resnet-18} & 99    & \textbf{94.2\%} & 2626.1 & {($1 \times$)} & 33672 & {($1 \times$)} \\ \cmidrule(l){2-7} 
                          & 10    & 93.1\% & 262.8 & {($9.9 \times$)} & 7957 & {($4.2 \times$)}\\ \cmidrule(l){2-7} 
                          & \adasp & \textbf{93.9\%} & 976.7 & {(\textbf{2.8}$\mathbf{\times}$)} & 9356 & {(\textbf{3.6}$\mathbf{\times}$)} \\ \midrule
\multirow{3}{*}{GoogleNet}  & 99   & \textbf{94.6\%} & 1430.9 & {($1 \times$)} & 28476 & {($1 \times$)} \\ \cmidrule(l){2-7} 
                          & 10    & 94.1\% & 145.2 & {($9.8 \times$)} & 13111 & {($2.1 \times$)} \\ \cmidrule(l){2-7}
                          & \adasp & \textbf{94.7\%} & 383.8 & {(\textbf{3.7}$\mathbf{\times}$)} & 16022 & {(\textbf{1.7}$\mathbf{\times}$)} \\ \midrule
\multirow{3}{*}{Senet}     & 99    & \textbf{94.6\%} & 2648.7 & {($1 \times$)} & 29977 & {($1 \times$)} \\ \cmidrule(l){2-7} 
                          & 10    & 93.8\% & 267.9 & {($9.8 \times$)} & 8055 & {($3.7 \times$)} \\ \cmidrule(l){2-7} 
                          & \adasp & \textbf{94.5\%} & 869.8 & {(\textbf{3.0}$\mathbf{\times}$)} & 13071 & {(\textbf{2.29}$\mathbf{\times}$)} \\ \bottomrule

\end{tabular}}
\end{minipage}\quad
\vspace{-0.1in}
\begin{minipage}{0.49\textwidth}
\caption{\small{\adasp\ using \topk\ on \cifarh\ }}
\label{tab:cifar100topk}
\resizebox{\linewidth}{!}{
\begin{tabular}{@{}lllllll@{}}
\toprule
\textbf{Network} &
  \textbf{K(\%)} &
  \textbf{Accuracy} &
  \multicolumn{2}{l}{\textbf{\begin{tabular}[c]{@{}l@{}}Data Sent\\(Billion Floats)\end{tabular}}} &
  \multicolumn{2}{l}{\textbf{\begin{tabular}[c]{@{}l@{}}Time\\ (Seconds)\end{tabular}}} \\ \midrule
\multirow{3}{*}{Resnet-18} & 99    & \textbf{72.4\%} & 2636.9 & {($1 \times$)} & 53460 & {($1 \times$)} \\ \cmidrule(l){2-7} 
                          & 25    & 71.3\% & 659.4 & {($3.9 \times$)} & 6176 & {($8.6 \times$)}\\ \cmidrule(l){2-7} 
                          & \adasp & \textbf{72.3\%} & 923.6 & {(\textbf{2.8}$\mathbf{\times}$)} & 14223 & {(\textbf{3.8}$\mathbf{\times}$)} \\ \midrule
\multirow{3}{*}{GoogleNet}  & 99   & \textbf{76.2\%} & 1452.4 & {($1 \times$)} & 28579 & {($1 \times$)} \\ \cmidrule(l){2-7} 
                          & 25    & 75.3\% & 367.3 & {($3.9 \times$)} & 12810 & {($2.23 \times$)} \\ \cmidrule(l){2-7}
                          & \adasp & \textbf{76.2\%} & 539.9 & {(\textbf{2.7}$\mathbf{\times}$)} & 15639 & {(\textbf{1.82}$\mathbf{\times}$)} \\ \midrule
\multirow{3}{*}{Senet}     & 99    & \textbf{72.8\%} & 2659.5 & {($1 \times$)} & 30312 & {($1 \times$)} \\ \cmidrule(l){2-7} 
                          & 25    & 71.9\% & 671.9 & {($3.9 \times$)} & 7376 & {($4.1 \times$)} \\ \cmidrule(l){2-7} 
                          & \adasp & \textbf{72.7\%} & 966.13 & {(\textbf{2.8}$\mathbf{\times}$)} & 10689 & {(\textbf{2.8}$\mathbf{\times}$)} \\ \bottomrule
                           
\end{tabular}}
\end{minipage}
\end{table*}
 In Table~\ref{tab:cifar10topk} and ~\ref{tab:cifar100topk} we show \adasp\ reaches the same accuracy as using \topk 99\% but with significantly less communication. Our implementation of TopK follows from~\citet{aji2017sparse}. We were unable to find details on parameters which work reasonably well for all networks. Thus, from a users perspective who wants performance extremely close to syncSGD we choose  \topk 99\% as low compression. For high compression we choose a value which provides significantly more compression. For ResNet-18 trained on \cifart\ we observe that high compression, \topk 25\% reaches accuracy of 71.3\% while low compression \topk 99\% reaches accuracy of 72.4\%, \adasp\ on the other hand reaches accuracy of 72.3\% while reducing the communication by $2.8\times$.  

\subsection{\adasp\ with  Large Batch size}
\begin{table*}[ht]
\begin{minipage}{0.485\textwidth}
\caption{\small{\adasp\ switching Batch Size on \cifart\ } }
\label{tab:cifar10batch}
\resizebox{\linewidth}{!}{
\begin{tabular}{@{}lllllll@{}}
\toprule
\textbf{Network} &
  \textbf{Batch Size} &
  \textbf{Accuracy} &
  \multicolumn{2}{l}{\textbf{\begin{tabular}[c]{@{}l@{}}Data Sent\\(Billion Floats)\end{tabular}}} &
  \multicolumn{2}{l}{\textbf{\begin{tabular}[c]{@{}l@{}}Time\\ (Seconds)\end{tabular}}} \\ \midrule
\multirow{3}{*}{Resnet-18} & 512    & \textbf{94.5\%} & 326.5 & {($1 \times$)} & 5009 & {($1 \times$)} \\ \cmidrule(l){2-7} 
                          & 4096    & 93.2\% & 40.22 & {($8 \times$)} & 1721 & {($2.9 \times$)}\\ \cmidrule(l){2-7} 
                          & \adasp & \textbf{94.4\%} & 59.22 & {(\textbf{5.6}$\mathbf{\times}$)} & 1959 & {(\textbf{2.5}$\mathbf{\times}$)} \\ \midrule
\multirow{3}{*}{GoogLeNet}  & 512   & \textbf{94.7\%} & 181.28 & {($1 \times$)} & 12449 & {($1 \times$)} \\ \cmidrule(l){2-7} 
                          & 4096    & 93.1\% & 22.19 & {($8.1 \times$)} & 3386 & {($3.67 \times$)} \\ \cmidrule(l){2-7}
                          & \adasp & \textbf{94.7\%} & 32.68 & {(\textbf{5.5}$\mathbf{\times}$)} & 6220 & {(\textbf{2.0}$\mathbf{\times}$)} \\ \midrule
\multirow{3}{*}{DenseNet}     & 512    & \textbf{93.9\%} & 29.4 & {($1 \times$)} & 14489 & {($1 \times$)} \\ \cmidrule(l){2-7} 
                          & 4096    & 93.1\% & 3.6 & {($8.1 \times$)} & 2759 & {($5.2 \times$)} \\ \cmidrule(l){2-7} 
                          & \adasp & \textbf{94.0\%} & 5.3 & {(\textbf{5.5}$\mathbf{\times}$)} & 3547 & {(\textbf{4}$\mathbf{\times}$)} \\ \bottomrule

\end{tabular}}                           
\end{minipage}\quad
\begin{minipage}{0.485\textwidth}
\caption{\small{\adasp\ switching Batch Size on \cifarh\ }}
\label{tab:cifar100batch}
\resizebox{\linewidth}{!}{
\begin{tabular}{@{}lllllll@{}}
\toprule
\textbf{Network} &
  \textbf{Batch Size} &
  \textbf{Accuracy} &
  \multicolumn{2}{l}{\textbf{\begin{tabular}[c]{@{}l@{}}Data Sent\\(Billion Floats)\end{tabular}}} &
  \multicolumn{2}{l}{\textbf{\begin{tabular}[c]{@{}l@{}}Time\\ (Seconds)\end{tabular}}} \\ \midrule
\multirow{3}{*}{Resnet-18} & 512    & \textbf{73.1\%} & 326.5 & {($1 \times$)} & 5096 & {($1 \times$)} \\ \cmidrule(l){2-7} 
                          & 4096    & 70.0\% & 40.39 & {($8 \times$)} & 1635 & {($3.1 \times$)}\\ \cmidrule(l){2-7} 
                          & \adasp & \textbf{73.3\%} & 54.96 & {(\textbf{5.5}$\mathbf{\times}$)} & 1852 & {(\textbf{2.7}$\mathbf{\times}$)} \\ \midrule
\multirow{3}{*}{GoogleNet}  & 512   & \textbf{77.0\%} & 182.1 & {($1 \times$)} & 12443 & {($1 \times$)} \\ \cmidrule(l){2-7} 
                          & 4096    & 73.7\% & 22.5 & {($8.1\times$)} & 5755 & {($2.1 \times$)} \\ \cmidrule(l){2-7}
                          & \adasp & \textbf{77.0\%} & 33.1 & {(\textbf{5.4}$\mathbf{\times}$)} & 6228 & {(\textbf{2.0}$\mathbf{\times}$)} \\ \midrule
\multirow{3}{*}{DenseNet}     & 512    & \textbf{73.7\%} & 30.126 & {($1 \times$)} & 14928 & {($1 \times$)} \\ \cmidrule(l){2-7} 
                          & 4096    & 70.0\% & 3.72 & {($8 \times$)} & 2775 & {($5.3 \times$)} \\ \cmidrule(l){2-7} 
                          & \adasp & \textbf{73.9\%} & 5.48 & {(\textbf{5.4}$\mathbf{\times}$)} & 3585 & {(\textbf{4.1}$\mathbf{\times}$)} \\ \bottomrule
                           
\end{tabular}}
\end{minipage}
\end{table*}
In Table~\ref{tab:cifar10batch} and~\ref{tab:cifar100batch} we show that \adasp\ is able to reach the same accuracy as small batch training without any hyper-parameter tuning. We modified no other parameter except scaling learning rate when switching Batch Size as described in Section~\ref{sec:hyper_param}. \adasp\ by switching between batch size of 512 and 4096 is able to save around $5.5\times$ in communications and up to $4.1\times$ reduction in wall clock training time. 

\subsection{Comparison with Prior Work}
We compare \adasp\ with prior work in adaptive gradient compression and adaptive batch size tuning. For adaptive gradient compression we consider recent work by \citet{guo2020acc} that uses the mean to standard deviation ratio (MSDR) of the gradients. If they observe that MSDR has reduced by a certain amount(a hyper-parameter), they correspondingly reduce the compression ratio by half (i.e., switch to a more accurate gradient). We use this approach with \powersgd\ and our experiments in Figure~\ref{fig:icassp_compare} suggest that their switching scheme ends up requiring more communication and also leads to some loss in accuracy.

\begin{figure*}[t]
    \begin{center}
    \begin{subfigure}[b]{0.45\textwidth}
        \includegraphics[width=\textwidth]{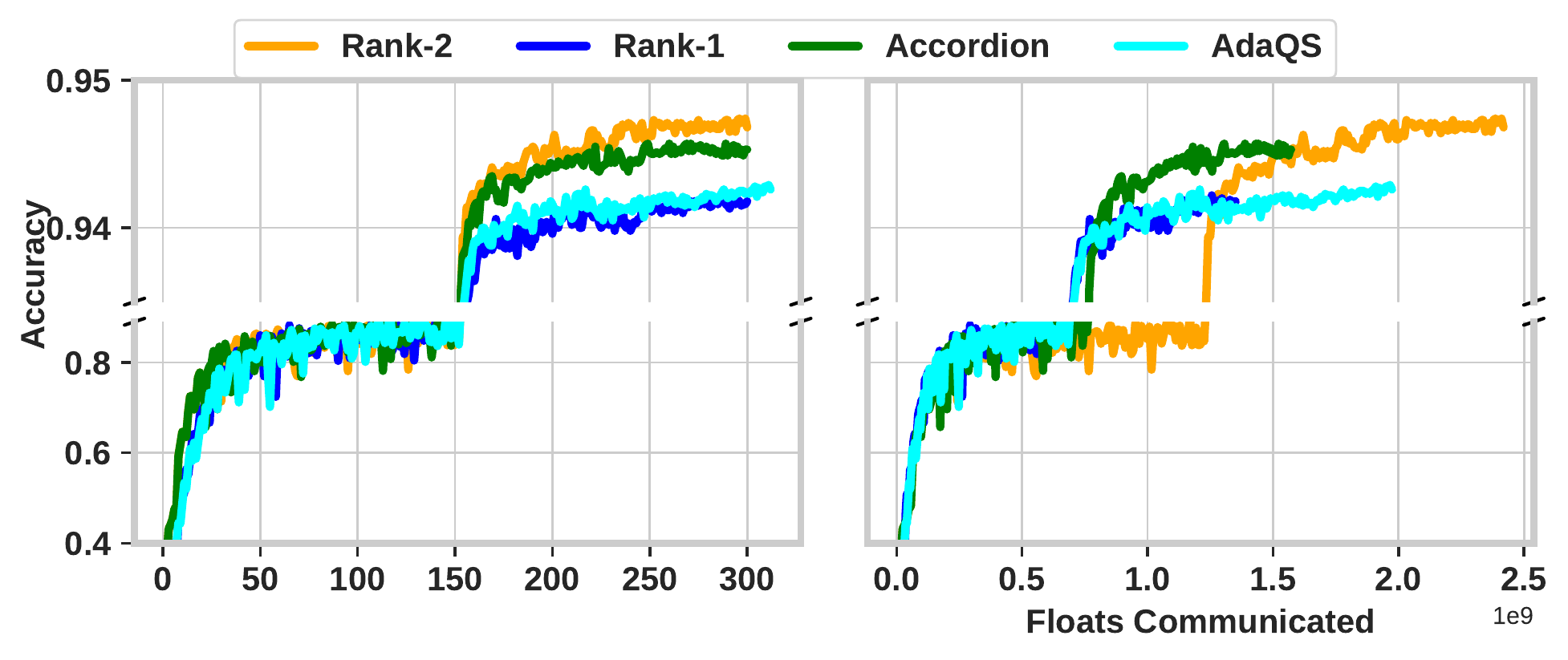}
        \caption{ResNet-18 trained on \cifart}
        \vspace{-0.1in}

        \label{fig:icasspcomparison}
        \end{subfigure}
    \begin{subfigure}[b]{0.45\textwidth}
        \includegraphics[width=\textwidth]{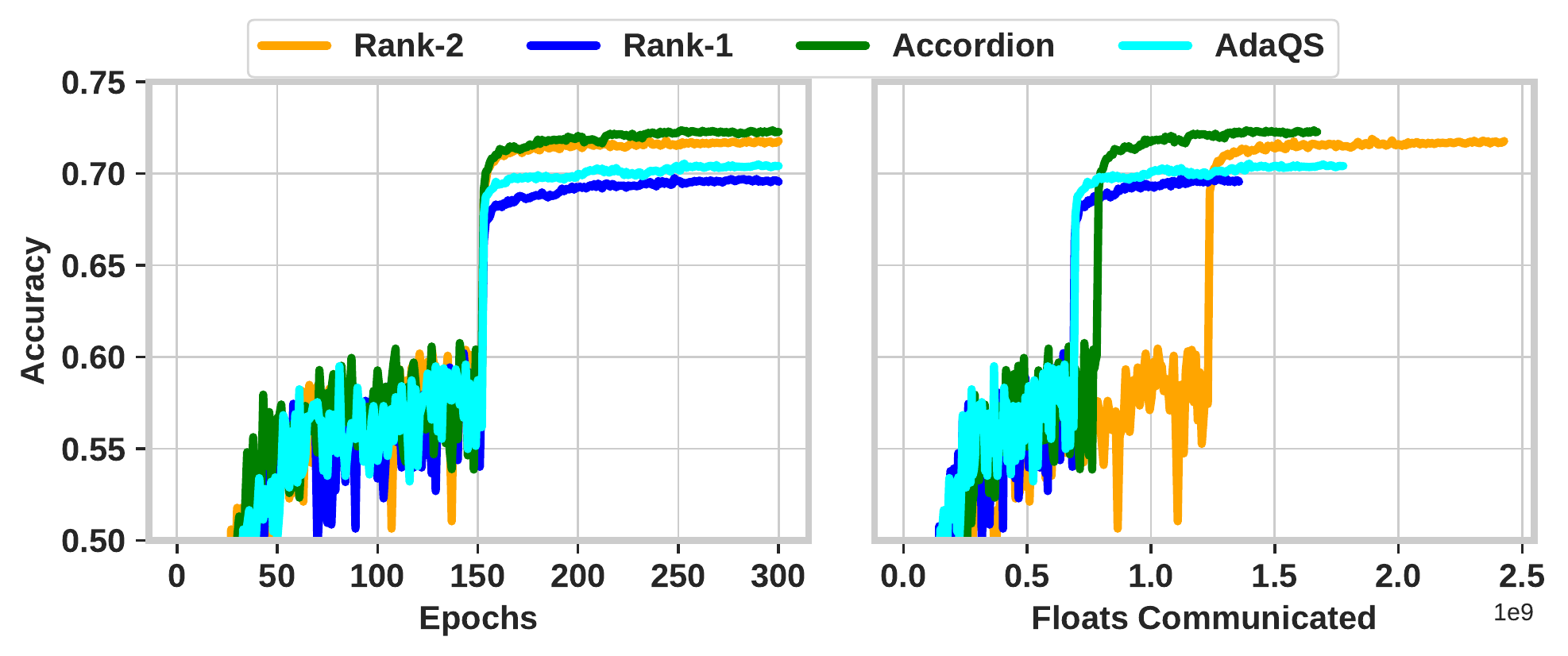}
        \caption{ResNet-18 trained of \cifarh}
        \vspace{-0.1in}
        \label{fig:cifarhicasspcompare}
    \end{subfigure}
    \end{center}
    \caption{\small{\textbf{Comparison with AdaQS:} We compare \adasp\ against AdaQS~\cite{guo2020acc} on \cifart\ and \cifarh. We use \powersgd as the Gradient Compressor. Even though AdaQS communicates more than \adasp\ it still looses accuracy compared to low compression. \adasp\ on the other hand with less communication is able to reach the accuracy of low compression. }}
    \label{fig:icassp_compare}
    \vspace{-0.1in}
\end{figure*}

For batch size we compare to ~\cite{smith2017don} in ~\Cref{fig:smith_compare}. We used the exact same setup as suggested by~\cite{smith2017don} and we use the \textit{Increased Initial Learning Rate} setting as shown in Figure 5 of their paper. We observe that \adasp\ reduces communication by $5.4\times$. On the other hand~\citet{smith2017don} only reduce communication by $2.2\times$. For \cifarh\ as shown in ~\Cref{fig:cifarhsmithcompare}  we observe that the approach presented by \citet{smith2017don} doesn't yield the same accuracy as small batch training. 
\begin{figure*}[t]
    \begin{center}
    \begin{subfigure}[b]{0.45\textwidth}
        \includegraphics[width=\textwidth]{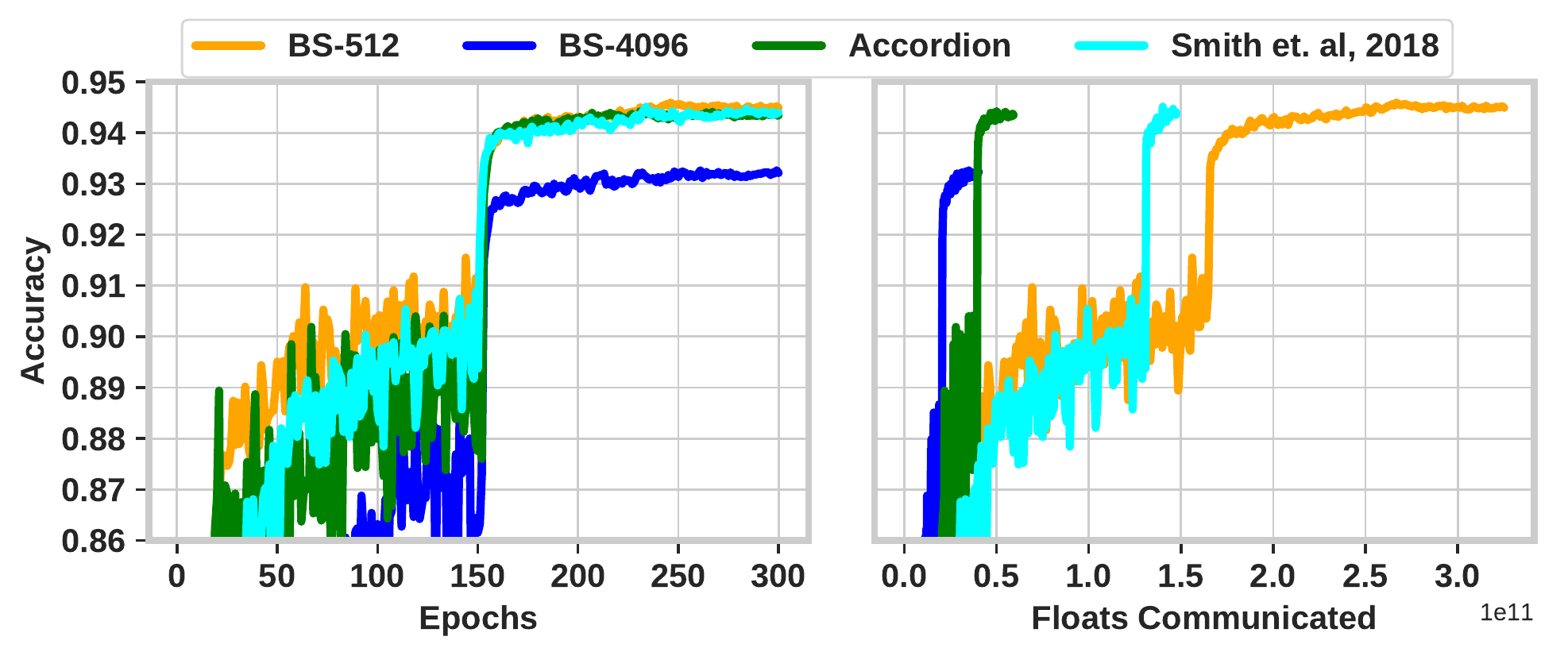}
        \caption{ResNet-18, trained on \cifart}
        \vspace{-0.1in}
        \label{fig:cifarsmithcomp}
    \end{subfigure}
    \begin{subfigure}[b]{0.45\textwidth}
        \includegraphics[width=\textwidth]{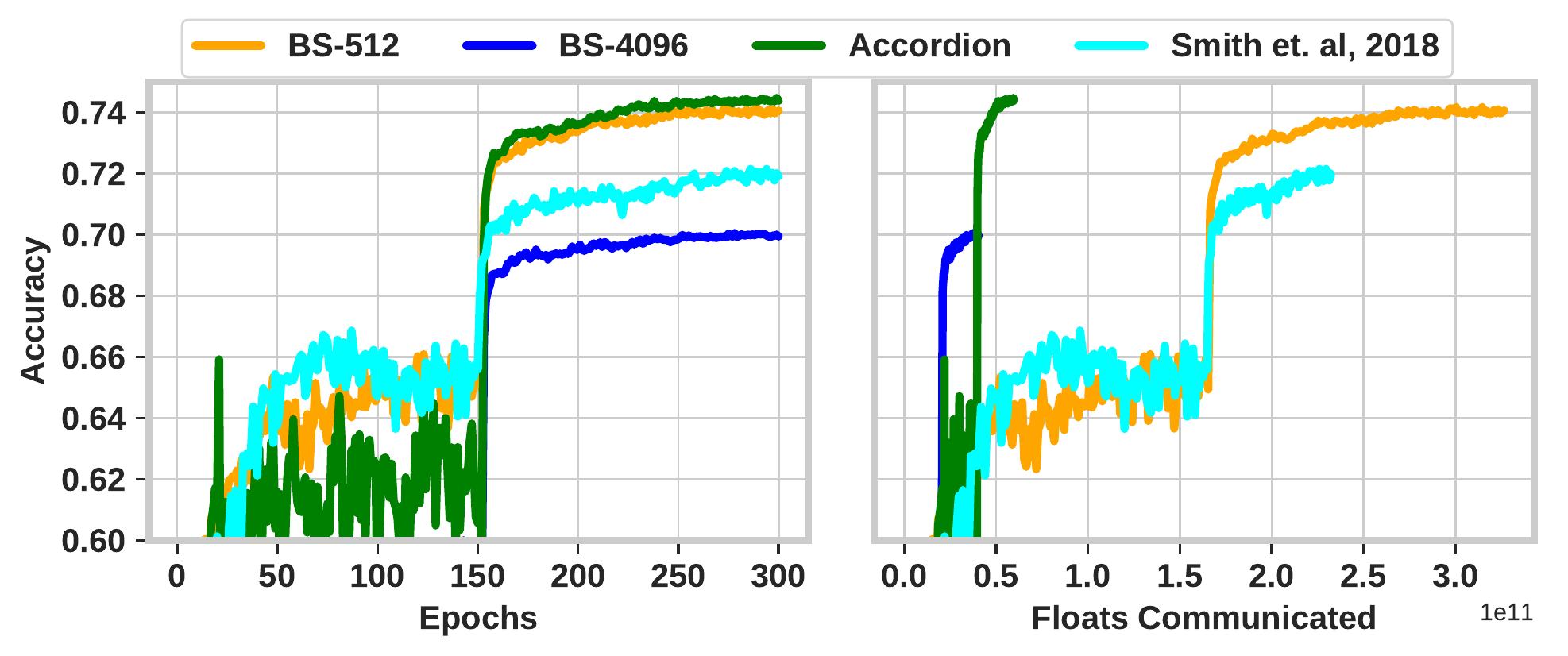}
        \caption{ResNet-18 trained on \cifarh}
        \vspace{-0.1in}
        \label{fig:cifarhsmithcompare}
    \end{subfigure}
    \end{center}
    \caption{\small{\textbf{\adasp\ compared with scheme proposed by ~\citet{smith2017don}:} We observe \adasp\ communicates around $3\times$ less than the scheme proposed by~\citet{smith2017don}. Moreover \adasp\ for both \cifarh\ and \cifart\ maintains same accuracy as using small batch size (high communication)}}
    \label{fig:smith_compare}
    \vspace{-0.1in}
\end{figure*}
Previous work~\cite{alistarh2017qsgd} has shown in theory that highly compressed gradients can reach the same accuracy as low compressed gradients when trained long enough. However, it only makes sense to run high compression if it can reach the same accuracy as low compression while communicating fewer bytes. To test this we ran ResNet-18~\cite{he2016deep} with \powersgd\ Rank-1 and Rank-2. We ran Rank-2 for 300 epochs and allowed Rank-1 to communicate the same amount as Rank-2. As observed in Figure~\ref{fig:equalfloats}, \powersgd\ Rank-1 cannot reach the same accuracy as \powersgd\ Rank-2. Moreover \adasp\ still achieves performance at par with low compression, while using a smaller communication budget.
\begin{figure}[t]
    \centering
    \begin{minipage}{0.485\textwidth}
    \centering
    \includegraphics[width=0.999\linewidth]{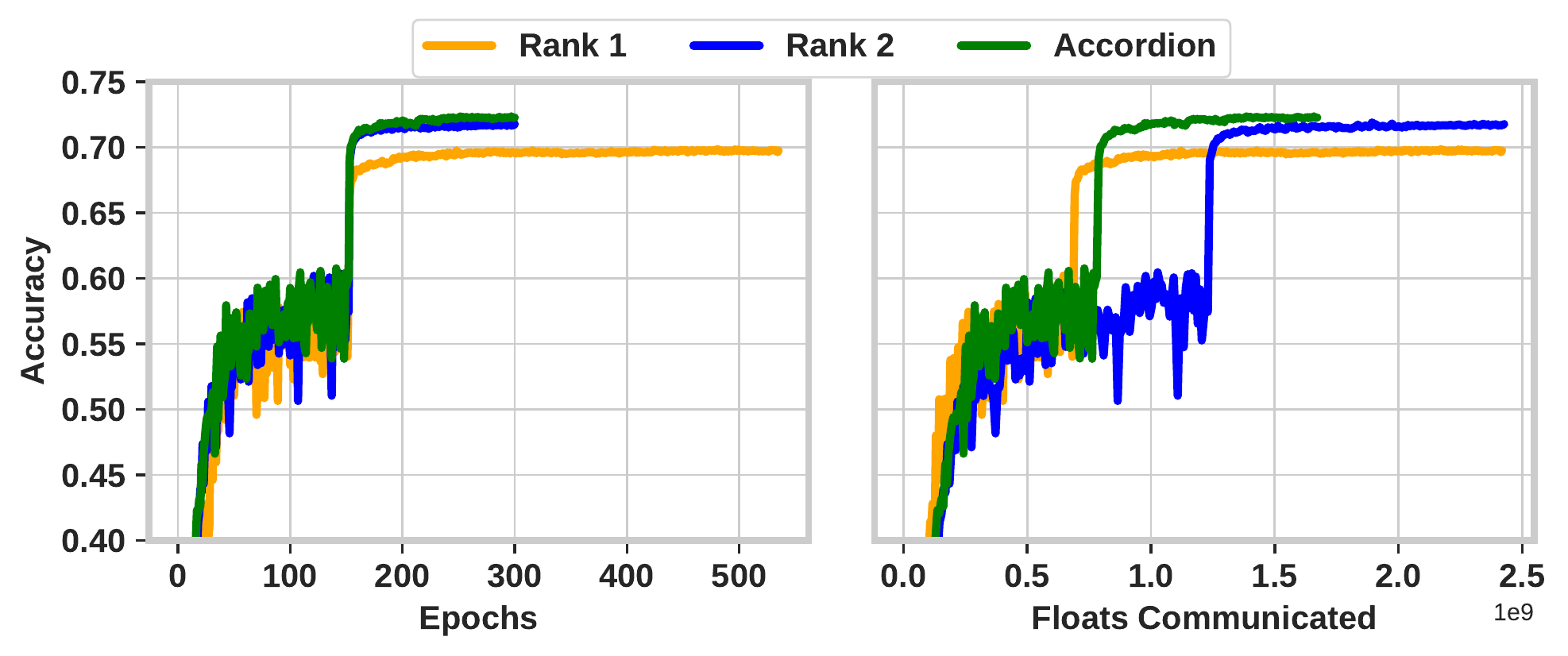}
    \caption{\small{\textbf{Evaluating high compression(Rank-1) training when allowed same communication budget as low compression(Rank-2):}ResNet-18 trained on \cifarh, using \powersgd. we observe that even when we allow highly compressed training to communicate same amount as low compressed training it is still not possible to get the same accuracy. Meanwhile \adasp\ is able to achieve the same accuracy as low compressed training but with much lesser communication 
    }}
    \label{fig:equalfloats}
    \end{minipage}\quad
    \begin{minipage}{0.485\textwidth}
    \centering
    \includegraphics[width=0.9\linewidth]{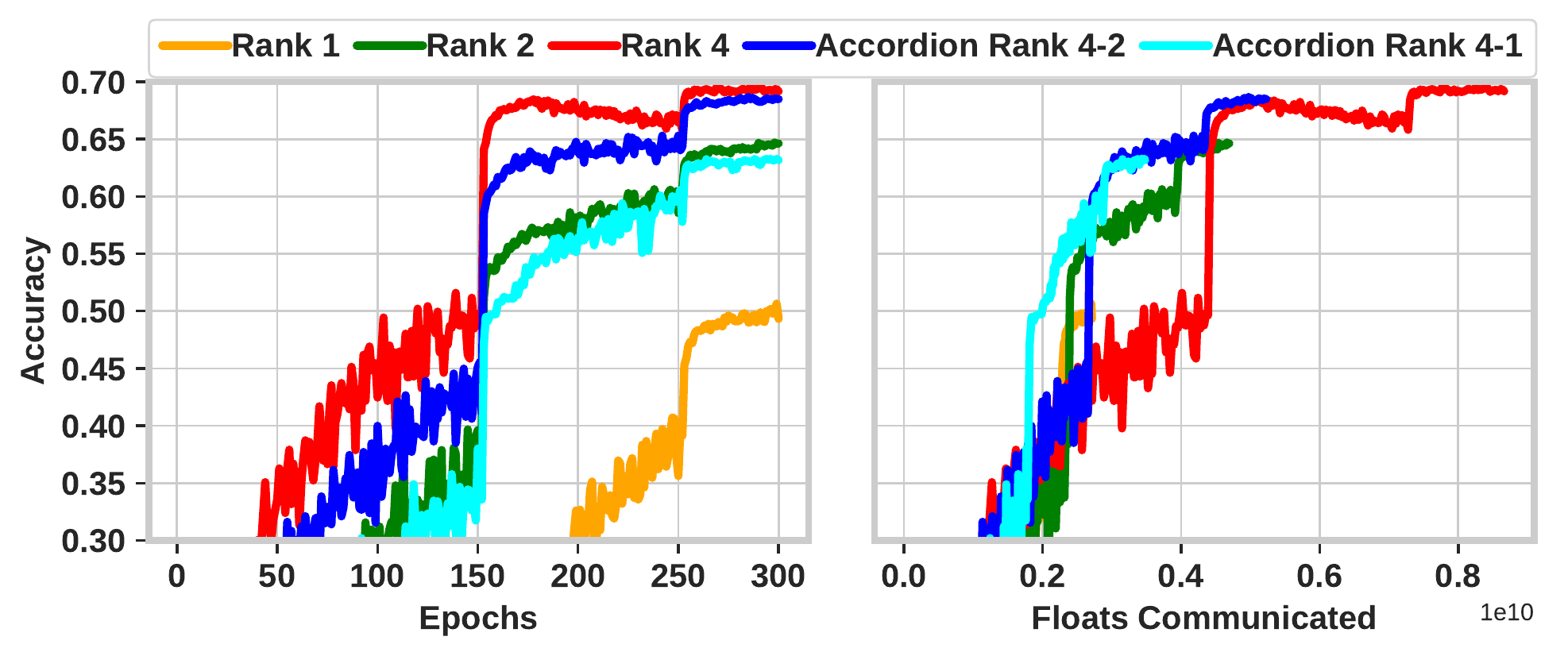}
    \caption{\small{\textbf{Limitation of \adasp:} For the specific case of VGG-19 trained on \cifarh, we observe that when \adasp\ switches between Rank-1(50\% accuracy) and Rank-4(68\% accuracy)  it only reaches accuracy of Rank-2(63\% accuracy) but when allowed to switch between Rank-2 and Rank-4 it reaches the accuracy of Rank-4 with around $2.3\times$ lesser communication. This shows that oftentimes the $\ell_{low}$ needs to be chosen carefully.}}
    \label{fig:limitation}
    \end{minipage}
\end{figure}

\section{Future Work and Limitations}
We have shown that \adasp\ provides significant benefits over static compression schemes without compromising accuracy or requiring additional hyper-parameter tuning. Next, we would like to point out some of the future directions and current limitations of our approach.
\begin{itemize}
    \item \textbf{Theoretical Understanding of Critical Regimes:} Although our work is motivated by several previous  works~\cite{jastrzebski2018on, Jastrzebski2020The, achille2018critical} which have discovered and analyzed critical regimes, building a better theoretical understanding of how change in gradient norm relates to critical regimes is an avenue for future work. 
    \item \textbf{Equivalence between batch size and gradient compression:} While we have shown that there might be a connection between batch size and gradient compression, rigorously verifying this connection can lead to better theoretical understanding of our technique.
    \item{\textbf{Choosing $\ell_{low}$, $\ell_{high}$, $B_{low}$ and $B_{high}$}:} Choosing the low and high compression ratios used by \adasp\ is currently left to the user. In case of \powersgd\ we chose $\ell_{low}$ based on the results of \citet{vogels2019powersgd} where authors showed Rank 2 and 4 achieved the same accuracy as syncSGD, making Rank 1 the natural choice for $\ell_{high}$. Similarly for TopK we chose $\ell_{low}$ to be close to SGD and $\ell_{high}$ to provide significant communication saving. However these settings do not work for all models. For example, in case of VGG19 on Cifar100 with \powersgd\ we observed that using $\ell_{high} = Rank 1$ leads to a model with very low accuracy(50\%). In that case \adasp\ cannot match the accuracy of $\ell_{low} = Rank 4$ as shown in Figure~\ref{fig:limitation}. Automating these choices has the potential of making gradient compression techniques much more user friendly and is an avenue for future work. 
    \item{\textbf{Jointly adapting batch size and gradient compression}} In this work, we study gradient compression and batch size scaling independently. Understanding how to vary both of them in tandem might lead to even large gains in the future. 
\end{itemize}
\section{Conclusion}
In this paper we propose \adasp, an adaptive gradient compression method that can automatically switch between low and high compression.
\adasp\ works by choosing low compression in critical regimes of training and high compression elsewhere. We show that such regimes can be efficiently identified using the rate of change of the gradient norm and that our method matches critical regimes identified by prior work. We also discuss connections between the compression ratio and batch size used for training and show that the insights used in \adasp\ are supported by prior work in adaptive batch size tuning. Finally, we show that \adasp\ is effective in practice and can save upto 3.7$\times$ communication compared to using low compression without affecting generalization performance.
Overall, our work provides a new principled approach for building adaptive-hyperparameter tuning algorithms, and we believe that further understanding of critical regimes in neural network training can help us design better hyperparameter tuning algorithms in the future.




\bibliographystyle{abbrvnat}

\bibliography{ref}

\appendix
\section{Detailed Experimental Settings}
\label{sec:experimental_details}
\paragraph*{Data preprocessing}
For preprocessing the images of \cifart\ and \cifarh\ datasets, we follow the standard data augmentation and normalization process. For data augmentation, random cropping and horizontal random flipping are used. Each color channel is normalized with it's mean and standard deviation. Where $\mu_r = 0.49, \mu_g = 0.48, \mu_b =  0.45$ are mean of the red, green and blue channels respectively. And $\sigma_r = 0.25, \sigma_g = 0.24, \sigma_b = 0.26$ are corresponding standard deviations. Each channel pixel is normalized by subtracting the mean value in this color channel and then divided by the standard deviation of this color channel. For pre-processing \wikitext\ we used the default english tokenizer in Spacy~\cite{spacy2}. 

\paragraph*{Hyperparameters}
\begin{table*}[h]
\caption{Hyperparameters}
\label{table:hyperparam}
\begin{tabularx}{\linewidth}{lXX}
    \toprule
    Dataset &\cifart\ and \cifarh  & \wikitext\\
    \midrule
    LR & 0.1 $\times$ Number of Workers.
         & 2.5 $\times$ Number of Workers\\
    LR Decay & $/10$ at epoch 150 and 250 & $/10$ at epoch 60 and 80\\ 
    LR warmup & Linearly for the first 5 epochs, starting from 0.1 & Linearly for the first 5 epochs, starting from 2.5\\
    Total Epochs & 300 & 90\\
    Optimizer & Nesterov & Nesterov \\
    Momentum & 0.9 & 0.9\\
    Repetition & 3 times with different random seeds & 3 times with different random seeds\\
    Error Bars & 95\% Confidence Interval & 95\% Confidence Interval\\
    \bottomrule
    
\end{tabularx}
\end{table*}

For training we used the standard hyper-parameters from prior work. We used \powersgd\ with memory term as suggested in~\cite{vogels2019powersgd} and used the same learning rate schedule. Table~\ref{table:hyperparam} provides details of the hyper-parameters used in our experiments. We used learning rate warmup as suggested by Goyal et al.~\cite{goyal2017accurate} for all our baselines as well as \adasp. We start with learning rate of 0.1 and linearly scale the learning rate 5 epochs to $0.1\times \frac{Batch Size}{128}$. 
\paragraph*{Additional Details for Batch Size experiment}
When trying to run extremely large batch sizes on 4 \textit{p3.2xlarge} we started running out of memory. To make sure that our communication overhead for each round remains same instead of using more GPU's, we simulated large batch size in Pytorch~\cite{paszke2019pytorch}. Which means we did multiple backward passes to accumulate the gradients before communicating and applying them to the weights.Moreover for training stability as done by ~\cite{yao2018large} we only allow \adasp\ to increase batch size.
\section{Connection Between Gradient Compression and Batch Size}
\label{sec:gradcompreesion_connection_batch_size}
The connection between gradient compression and batch size tuning can be made more formal under the following assumption: ``each stochastic gradient is the sum of a sparse mean and a dense noise'', i.e.,\
\begin{equation}
\begin{split}
\nabla_w \ell(w;x_i,y_i)& = \underbrace{\mathbb{E}_j \nabla_w \ell(w;x_j,y_j)}_{\text{sparse, large magnitudes}} +\underbrace{(\nabla_w \ell(w;x_i,y_i)-\mathbb{E}_j \nabla_w \ell(w;x_j,y_j))}_{\text{dense, small magnitudes}}
\end{split}
\end{equation}

Under this assumption, we can see that ``large batch gradient $\approx$ highly compressed gradient'', as a large batch gradient will be close to $\mathbb{E}_j \nabla_w \ell(w;x_j,y_j)$ by the law of large numbers, and a highly compressed gradient will also pick up the same sparse components. 
Similarly, a small batch gradient is equivalent to weakly compressed gradient. 

We show that the above assumption on gradient properties can hold for limited scenarios by considering a simple LASSO example.
Consider a model whose goal is to minimize $\sfrac{1}{2}\cdot\|Xw - y\|_2^2 + \lambda \|w\|$, where positive-class data points $x_{+} \sim \mathcal{N}(\mu, \sigma^2 I)$, negative-class data points $x_{-} \sim \mathcal{N}(-\mu, \sigma^2 I)$, and $P(Y=+1)=P(Y=-1)=\sfrac{1}{2}$.
Here, $w$ is sparse for a properly chosen value of $\lambda$ due to the shrinkage operation~\cite{lasso}.
Then, we have the following lemma, which implies that the gradient modeling described above holds w.h.p.

\begin{lemmat}
If $\mu$ is $k_1$-sparse and $w$ is $k_2$-sparse, $\mathbb{E}_j \nabla_w \ell(w;x_j,y_j)$ is $k_1 + k_2$-sparse. Let us denote by $\gamma$ the minimum absolute value of non-zero entries of $\mathbb{E}_j \nabla_w$.
Then, for any positive integer $n$ and $\epsilon > 0$, there exists a sufficiently small enough $\sigma > 0$ such that $\mathbb{P}(\max_{i \in [n]}\|\nabla_w \ell(w;x_i,y_i)-\mathbb{E}_j \nabla_w \ell(w;x_j,y_j)\|_{\infty} < \gamma) \geq 1-\epsilon$.
\end{lemmat}

\begin{proof}
We first show that $\mathbb{E}_j \nabla_w \ell(w;x_j,y_j)$ is $k_1 + k_2$-sparse. Since $\nabla_w \ell(w;x_i,y_i) = x_i (x_i^\top w) - x_i y_i + \lambda \text{sign}(w)$, we have
\begin{align*}
\mathbb{E} \nabla_w \ell(w;x_j,y_j) &= \mathbb{E} [x_i (x_i^\top w) - x_i y_i + \lambda \text{sign}(w)] \\&= \mathbb{E} [x_i x_i^\top] w + \lambda \text{sign}(w).
\end{align*}
Since $\mathbb{E}[x_i x_i^\top] = I + \mu\mu^\top$, 
\begin{align*}
\mathbb{E} \nabla_w \ell(w;x_j,y_j) &= (I + \mu \mu^\top) w + \lambda \text{sign}(w) \\&= w + \lambda \text{sign}(w) + \mu (\mu^\top w).
\end{align*}
The first two terms are $k_1$-sparse sharing the support. 
The sparsity of the last term is upper bounded by the that of $\mu$, hence $k_2$-sparse.

We now prove $\mathbb{P}(\max_{i \in [n]}\|\nabla_w \ell(w;x_i,y_i)-\mathbb{E}_j \nabla_w \ell(w;x_j,y_j)\|_{\infty} < \gamma) \geq 1-\epsilon$.
Note that it is sufficient to show $\mathbb{P}(\|\nabla_w \ell(w;x_i,y_i)-\mathbb{E}_j \nabla_w \ell(w;x_j,y_j)\|_{\infty} \geq \gamma) \leq \epsilon/n$ for all $i$.
Therefore, it is sufficient to show $\mathbb{P}((\nabla_w \ell(w;x_i,y_i)-\mathbb{E}_j \nabla_w \ell(w;x_j,y_j))_j \geq \gamma) \leq \epsilon/(nd)$ for all $i,j$, where $d$ is the dimension of the model parameter.
By Chebyshev's inequality, we have $\mathbb{P}((\nabla_w \ell(w;x_i,y_i)-\mathbb{E}_j \nabla_w \ell(w;x_j,y_j))_j \geq \gamma) \leq \frac{\mathrm{Var}[(\nabla_w \ell(w;x_i,y_i)_j]}{\gamma^2}$.
It is easy to see that $\mathrm{Var}[(\nabla_w \ell(w;x_i,y_i)_j] \leq (\sigma^4 + 2\|\mu\|_\text{max}^2 \sigma^2)\|w\|_2^2 + \sigma^2$, which converges to $0$ as $\sigma \rightarrow 0$.
Therefore, one can always find a small enough $\sigma > 0$ such that $\frac{(\sigma^4 + 2\|\mu\|_\text{max}^2 \sigma^2)\|w\|_2^2 + \sigma^2}{\gamma^2} = \frac{\epsilon}{nd}$, which is a sufficient condition for the desired inequality. 
\end{proof}

\section{\adasp\ on Extremely Large Batch Size}
To push the limits of Batch Size scaling further we tried using \adasp\ for scaling \cifart\ on ResNet-18 to batch size of 16,384.We observed that using \adasp\ looses around (1.6\%) accuracy compared to using batch size 512. Interestingly we also observe that when \adasp\ first switches the batch size there is a rapid drop, but then training immediately recovers. 
 \begin{figure}[hb]
     \centering
     \includegraphics[width=0.7\linewidth]{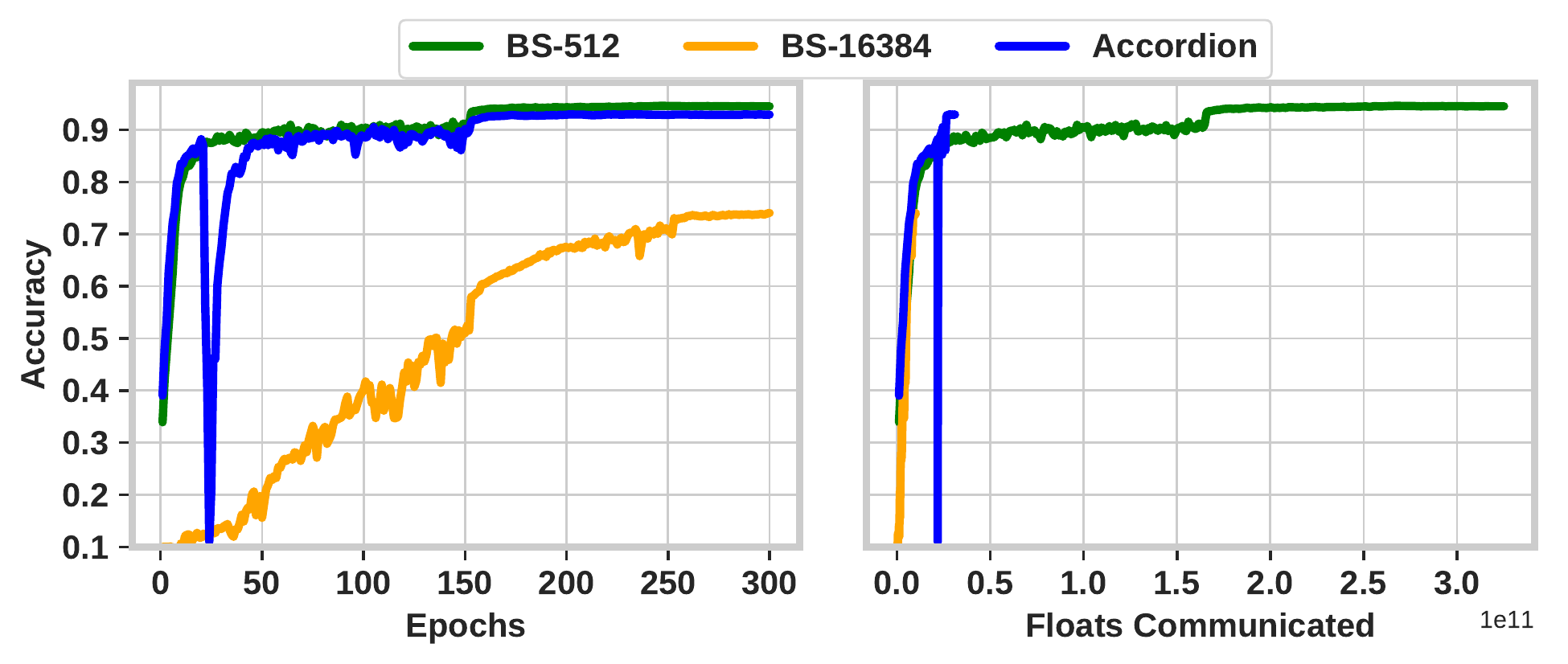}
     \caption{\small{\textbf{Using Extremely Large Batch Size:} We observe that \adasp\ looses around 1.6\% accuracy when we use batch size of 16,384. Showing \adasp\ can often prevent large accuracy losses while providing massive gains. }}
 \end{figure}

\section{Results and Detailed Analysis}
\label{app:addition_results}
We present detailed analysis with error bars for the results presented in ~\Cref{tab:cifar10psgd,tab:cifar100psgd,tab:cifar10topk,tab:cifar100topk}.
\subsection{Language Model}
\Cref{fig:topk_lstm} shows \adasp's performance for training a 2 Layer LSTM on \wikitext. By automatically switching between \topk99\% and \topk2\% \adasp\ is able to bridge the perplexity score and achieve the same accuracy as \topk99\% with significantly less overall communication.
\begin{figure*}[!htb]
    \begin{center}
         \includegraphics[width=\textwidth]{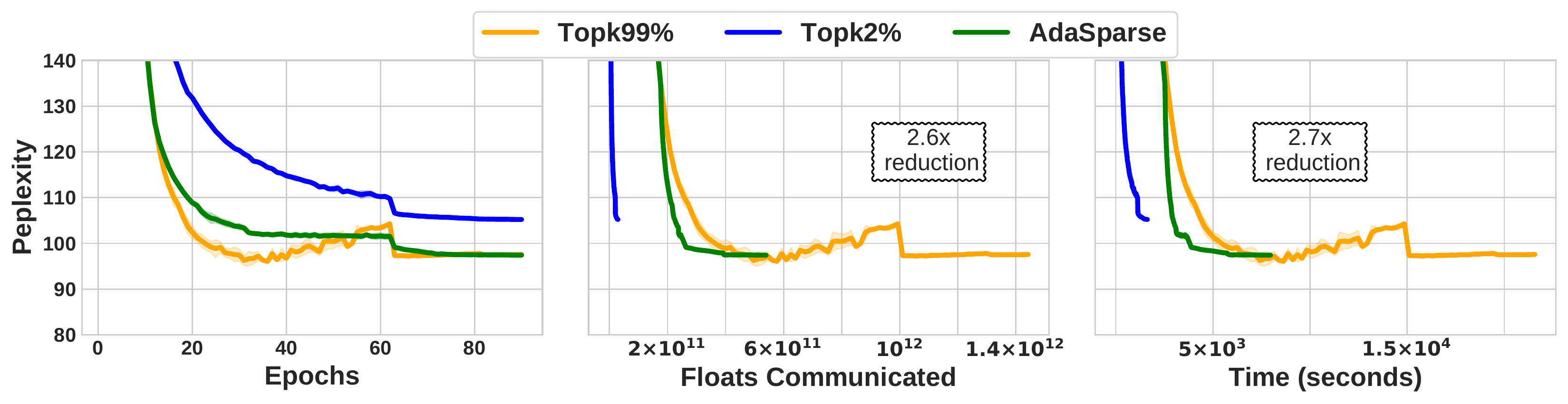}
    \end{center}
    \caption{\small{\textbf{\adasp\ using \topk\ with \elltopkl{99} and \elltopkh{2} for training LSTM on \wikitext} 
    (\textbf{left:}) Perplexity vs Epochs,  (\textbf{center:}) Perplexity vs Floats Communicated, (\textbf{right:}) Perplexity vs Time(seconds): \adasp\ significantly reduces total communication and training time compared to using \elltopkl{99} throughtout training}}
    \label{fig:topk_lstm} 
\end{figure*}

\subsection{Computer Vision Models}
We present graphs corresponding to the results stated in~\Cref{tab:cifar10psgd,tab:cifar100psgd,tab:cifar10topk,tab:cifar100topk} in the main text. In  ~\Cref{fig:powersgd_cifar10,fig:powersgd_cifar100,fig:topk_cifar10,fig:topk_cifar100} we provide details on \adasp's performance with error bars.
\begin{figure*}[!t]
    \caption{\small{\textbf{\adasp\ on Computer Vision Models trained on \cifart\ using \powersgd}}}
    \label{fig:powersgd_cifar10}
    \begin{center}
    \begin{subfigure}[b]{\textwidth}
        \includegraphics[width=\textwidth]{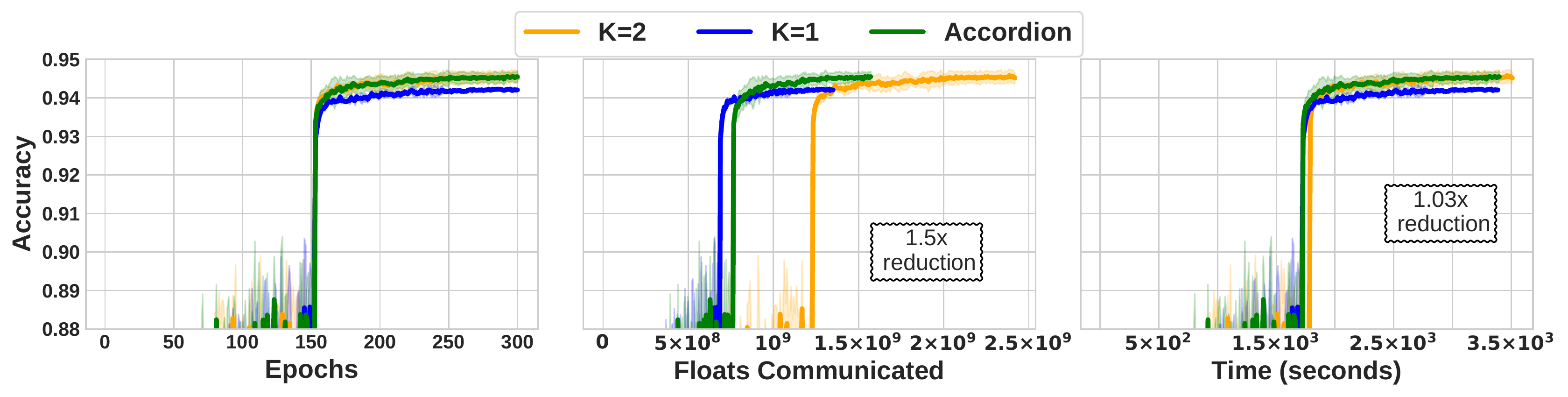}
        \caption{\small{{ResNet-18 trained using \powersgd\ with \ellpsgdl{4} and \ellpsgdh{1} for training}}}
        \label{fig:psgd:res18}
    \end{subfigure}
    \begin{subfigure}[b]{\textwidth}
        \includegraphics[width=\textwidth]{Figures/psgd_vgg19bn_cifar10_app.pdf}
        \caption{\small{{VGG-19bn trained using \powersgd\ with \ellpsgdl{4} and \ellpsgdh{1} for training}}}
    \end{subfigure}
     \begin{subfigure}[b]{\textwidth}
     \includegraphics[width=\textwidth]{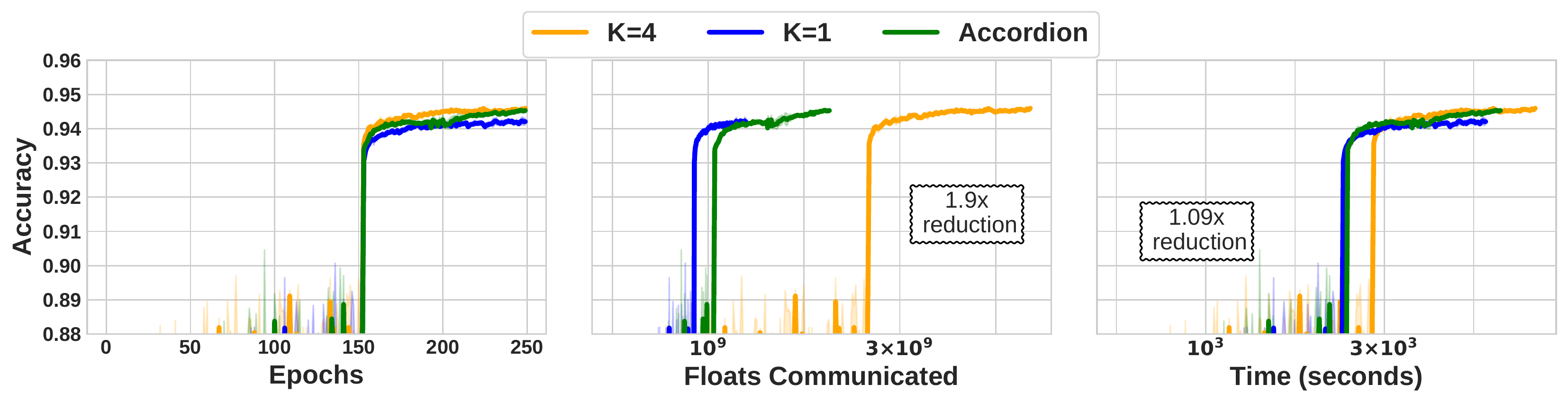}
      \caption{\small{{SeNet trained using \powersgd\ with \ellpsgdl{4} and \ellpsgdh{1} for training }}}
     \end{subfigure}
    \end{center}
\end{figure*}

\begin{figure*}[!t]
\caption{\small{\textbf{\adasp\ on Computer Vision Models trained on \cifarh\ using \powersgd}}}
 \label{fig:powersgd_cifar100}
    \begin{center}
    \begin{subfigure}[b]{\textwidth}
        \includegraphics[width=\textwidth]{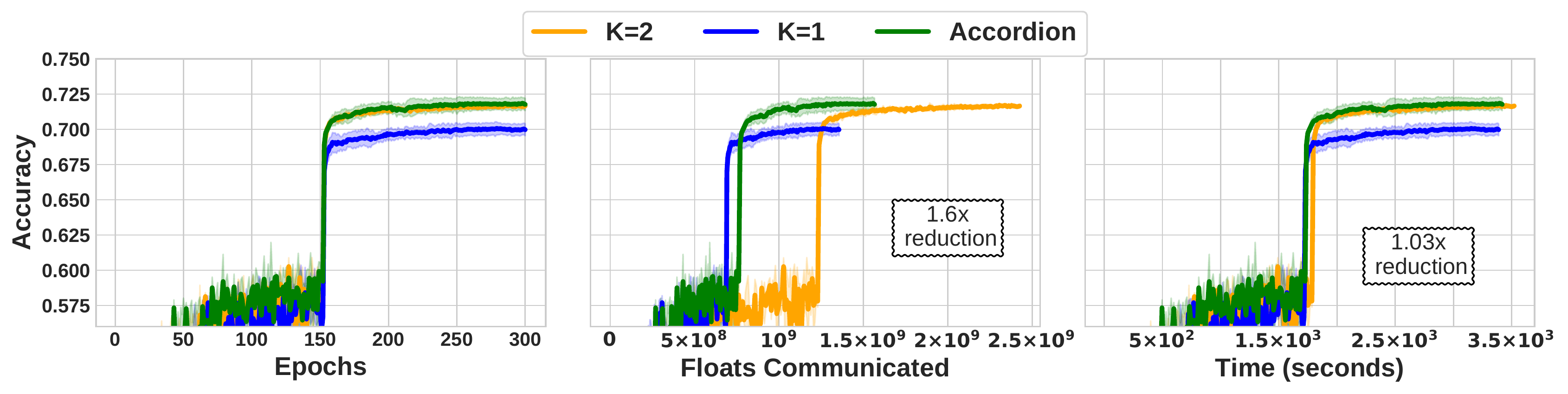}
        \caption{\small{{ResNet-18 trained using \powersgd\ with \ellpsgdl{2} and \ellpsgdh{1} for training }}}
    \end{subfigure}
    \begin{subfigure}[b]{\textwidth}
        \includegraphics[width=\textwidth]{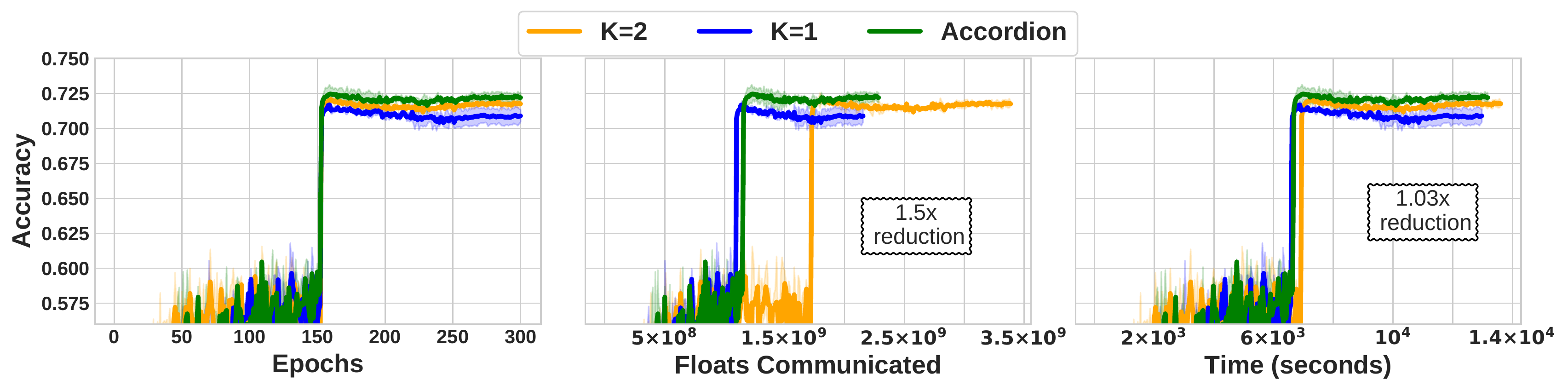}
        \caption{\small{{DenseNet trained using \powersgd\ with \ellpsgdl{2} and \ellpsgdh{1} for training }}}
    \end{subfigure}
    \begin{subfigure}[b]{\textwidth}
        \includegraphics[width=\textwidth]{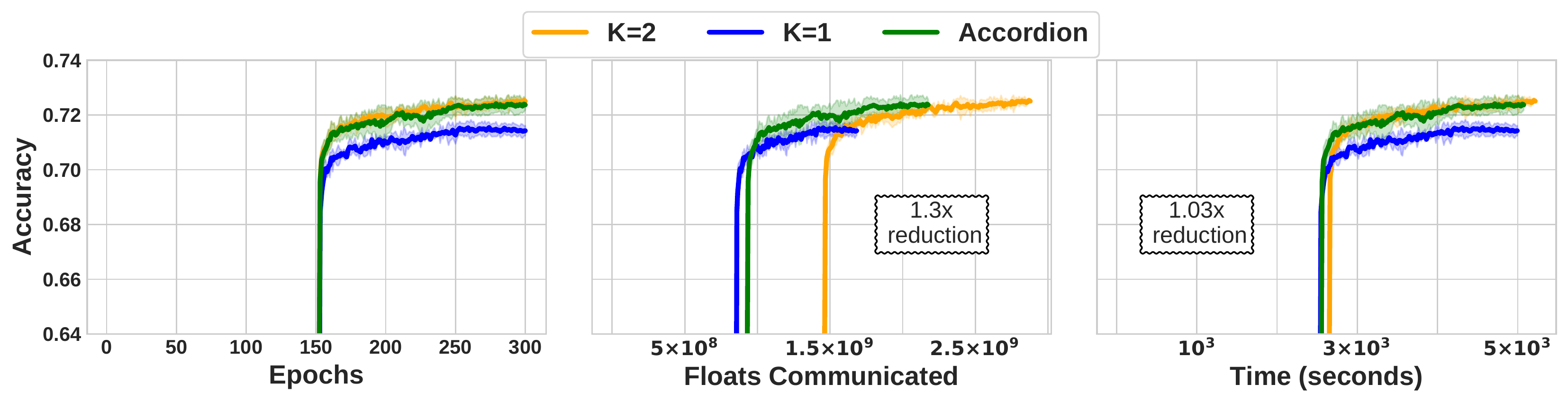}
        \caption{\small{{SeNet trained using \powersgd\ with \ellpsgdl{2} and \ellpsgdh{1} for training }}}
    \end{subfigure}
    \end{center}
\end{figure*}
\begin{figure*}[!t]
\caption{\small{\textbf{\adasp\ on Computer Vision Models trained on \cifart\ using \topk}}}
    \label{fig:topk_cifar10}
    \begin{center}
        \begin{subfigure}[b]{\textwidth}
            \includegraphics[width=\textwidth]{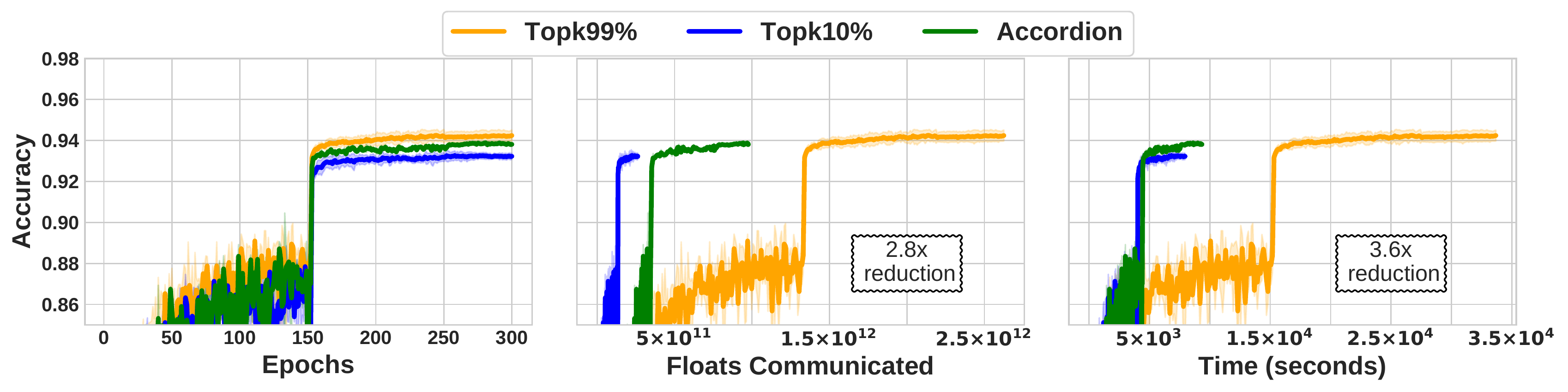}
            \caption{\small{{ResNet-18 trained using \topk\ with \elltopkl{99} and \elltopkh{10} for training }}}
        \end{subfigure}
        \begin{subfigure}[b]{\textwidth}
            \includegraphics[width=\textwidth]{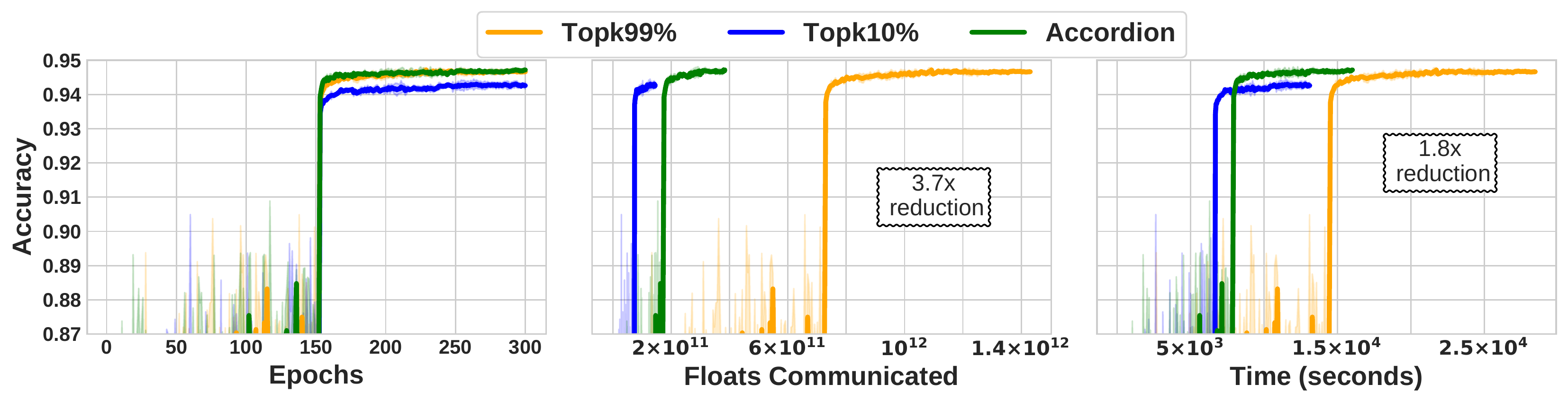}
            \caption{\small{{GoogLeNet trained using \topk\ with \elltopkl{99} and \elltopkh{10} for training }}}
        \end{subfigure}
         \begin{subfigure}[b]{\textwidth}
            \includegraphics[width=\textwidth]{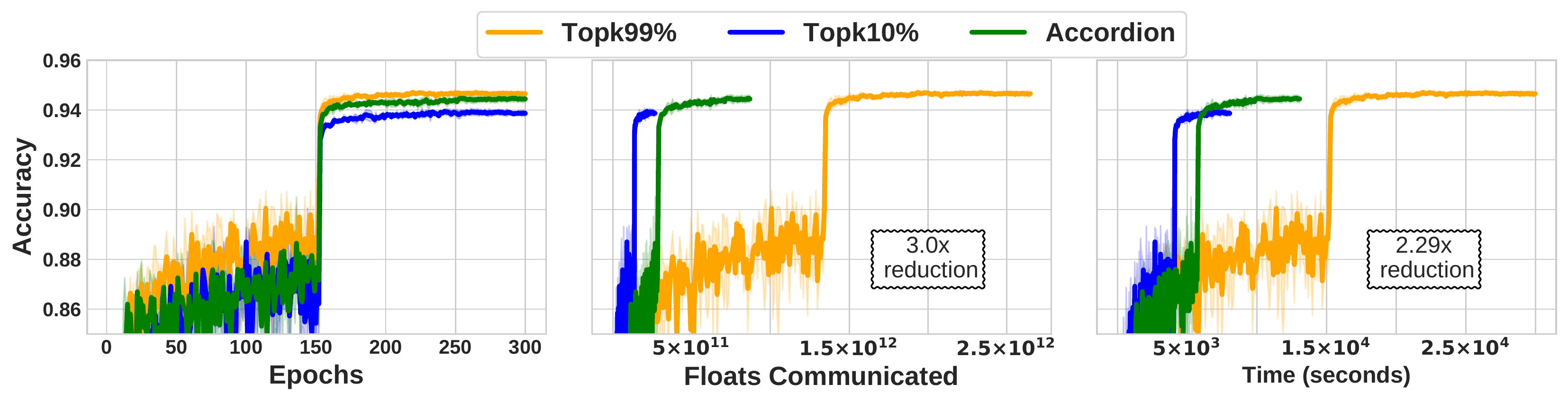}
            \caption{\small{{SeNet trained using \topk\ with \elltopkl{99} and \elltopkh{10} for training }}}
        \end{subfigure}
    \end{center}
\end{figure*}

\begin{figure*}[!t]
\caption{\small{\textbf{\adasp\ on Computer Vision Models trained on \cifarh\ using \topk:} }}
    \label{fig:topk_cifar100}
    \begin{center}
    \begin{subfigure}[b]{\textwidth}
        \includegraphics[width=\textwidth]{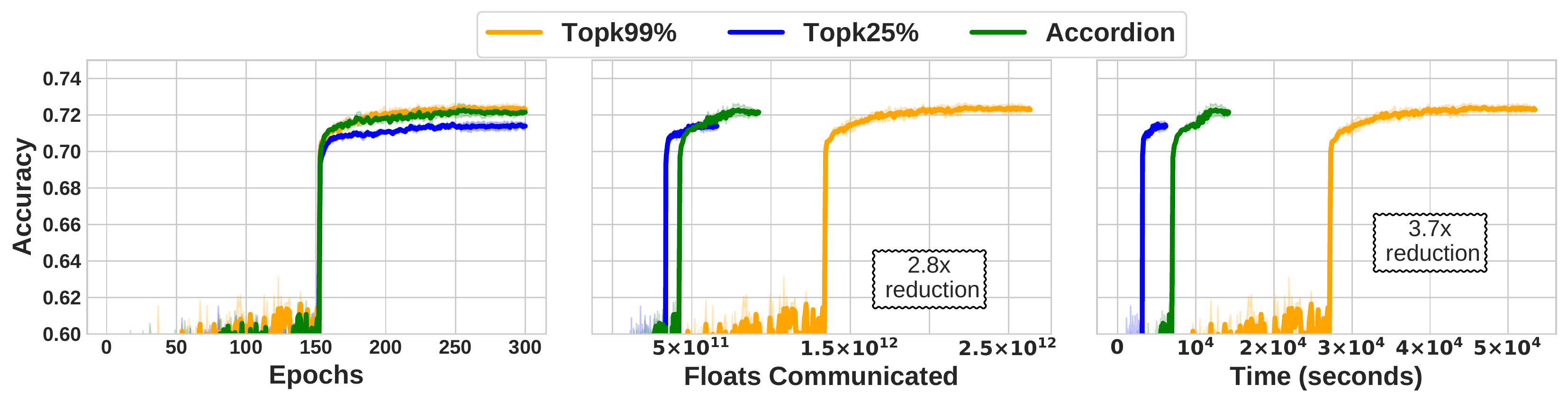}
        \caption{\small{{ResNet-18 trained using \topk\ with \elltopkl{99} and \elltopkh{25} for training }}}
    \end{subfigure}
     \begin{subfigure}[b]{\textwidth}
        \includegraphics[width=\textwidth]{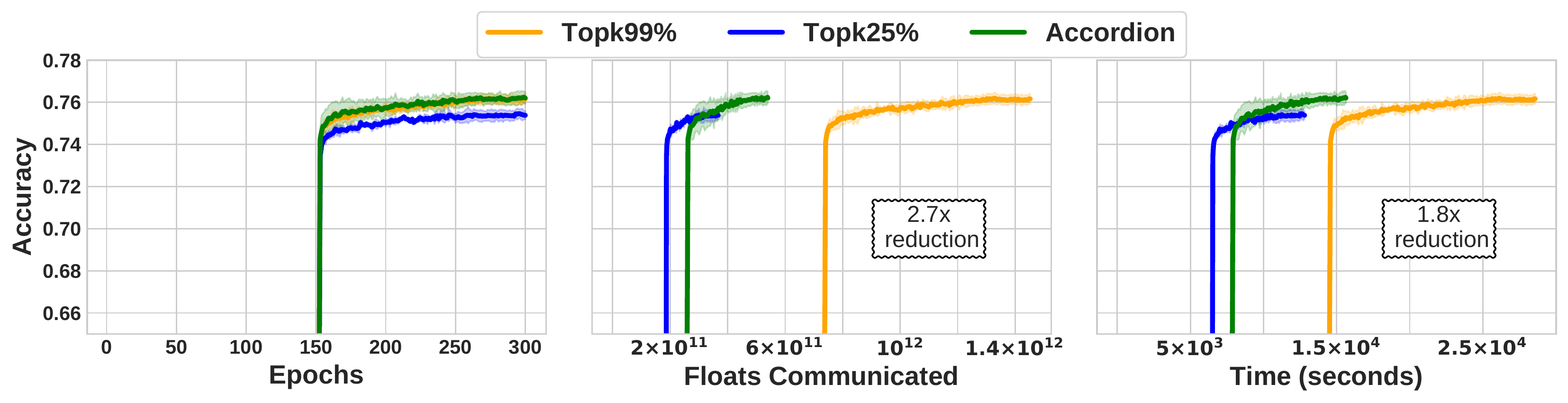}
        \caption{\small{{GoogLeNet trained using \topk\ with \elltopkl{99} and \elltopkh{25} for training }}}
    \end{subfigure}
    \begin{subfigure}[b]{\textwidth}
        \includegraphics[width=\textwidth]{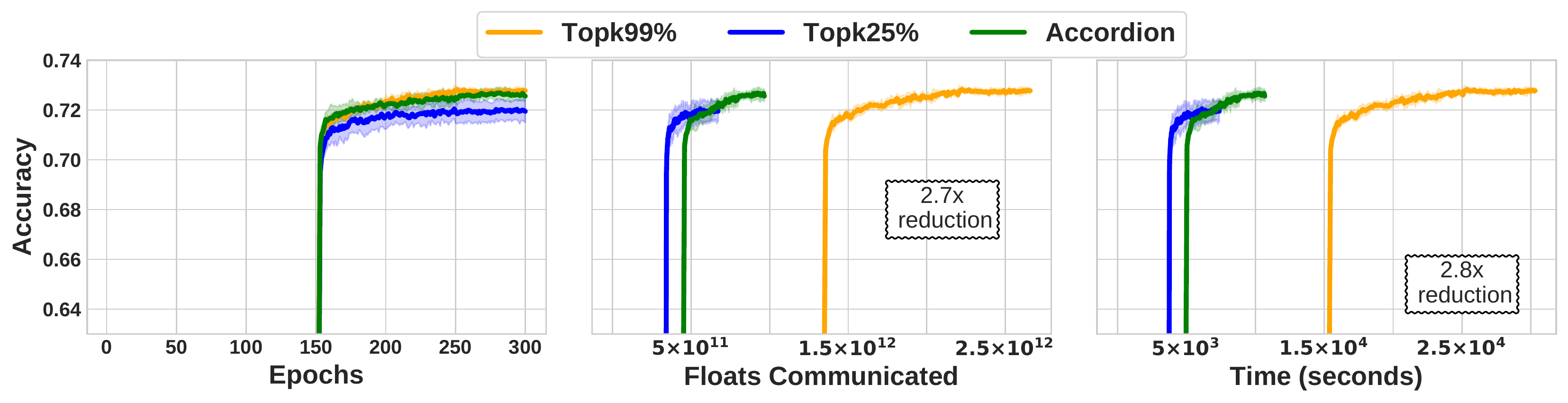}
        \caption{\small{{SeNet trained using \topk\ with \elltopkl{99} and \elltopkh{25} for training }}}
    \end{subfigure}
    \end{center}
\end{figure*}


\FloatBarrier
\section{Detailed Analysis of Batch Size Results}
In \Cref{fig:batchsize_cifar10} and~\ref{fig:batchsize_cifar100} we provide detailed analysis for batch size. We show that, we ran experiments on \cifart\ and \cifarh. For three different CNN's, two recent CNN's with skip connections(ResNet-18, DenseNet) and one CNN without skip connections (Inception V1). Based on the findings of ~\cite{goyal2017accurate, devarakonda2017adabatch} we also modify the learning rate in the same proportion as the change in batch size. 
\begin{figure*}[!t]
\caption{\small{\textbf{\adasp\ on Computer Vision Models trained on \cifart} }}
    \label{fig:batchsize_cifar10}
    \begin{center}
    \begin{subfigure}[b]{\textwidth}
        \includegraphics[width=\textwidth]{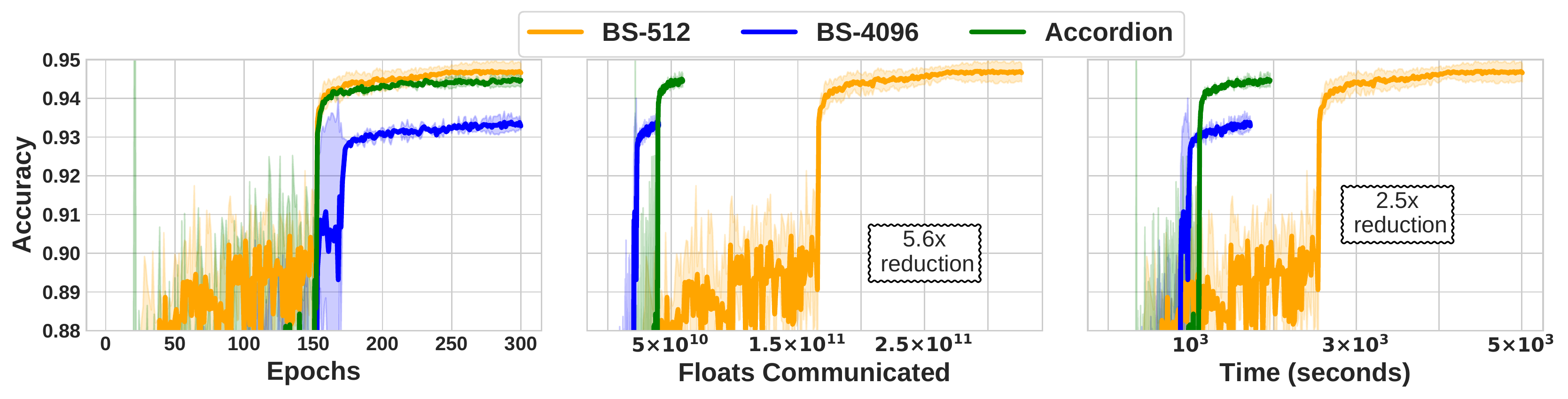}
        \caption{\small{{ResNet-18 trained using Batch Size=512,Batch Size=4096 and \adasp}}}
    \end{subfigure}
     \begin{subfigure}[b]{\textwidth}
        \includegraphics[width=\textwidth]{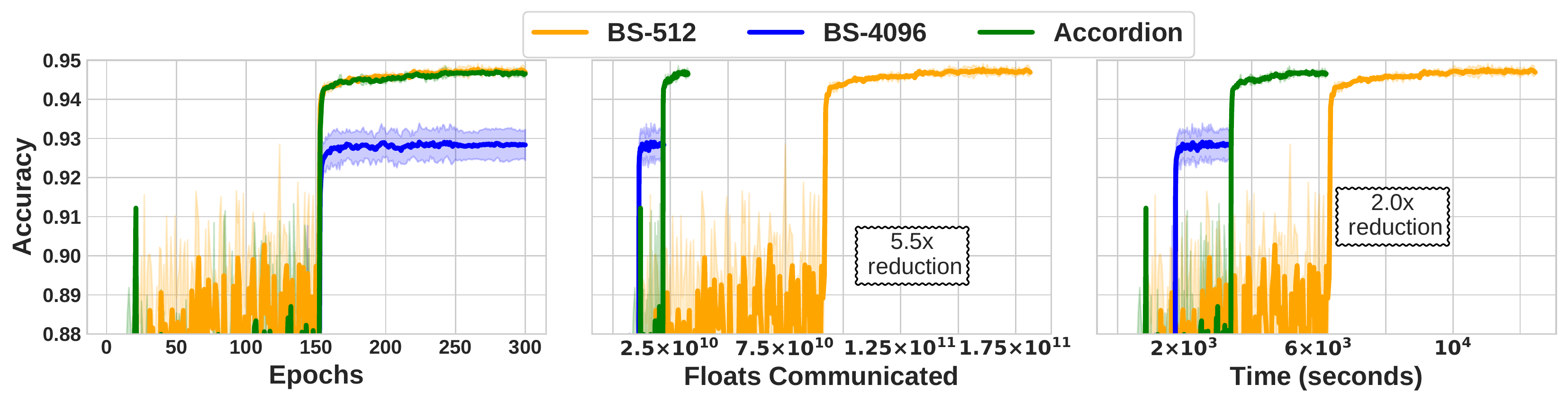}
        \caption{\small{{GoogLeNet trained using Batch Size=512,Batch Size=4096 and \adasp}}}
    \end{subfigure}
    \begin{subfigure}[b]{\textwidth}
        \includegraphics[width=\textwidth]{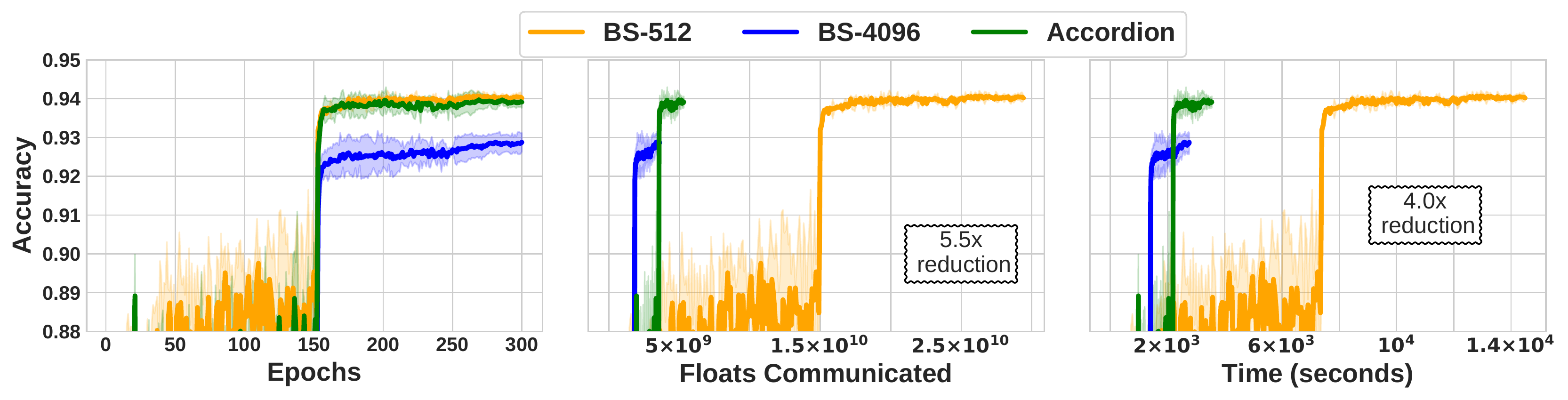}
        \caption{\small{{DenseNet trained using Batch Size=512,Batch Size=4096 and \adasp }}}
    \end{subfigure}
    \end{center}
\end{figure*}

\begin{figure*}[!t]
\caption{\small{\textbf{\adasp\ on Computer Vision Models trained on \cifarh} }}
    \label{fig:batchsize_cifar100}
    \begin{center}
    \begin{subfigure}[b]{\textwidth}
        \includegraphics[width=\textwidth]{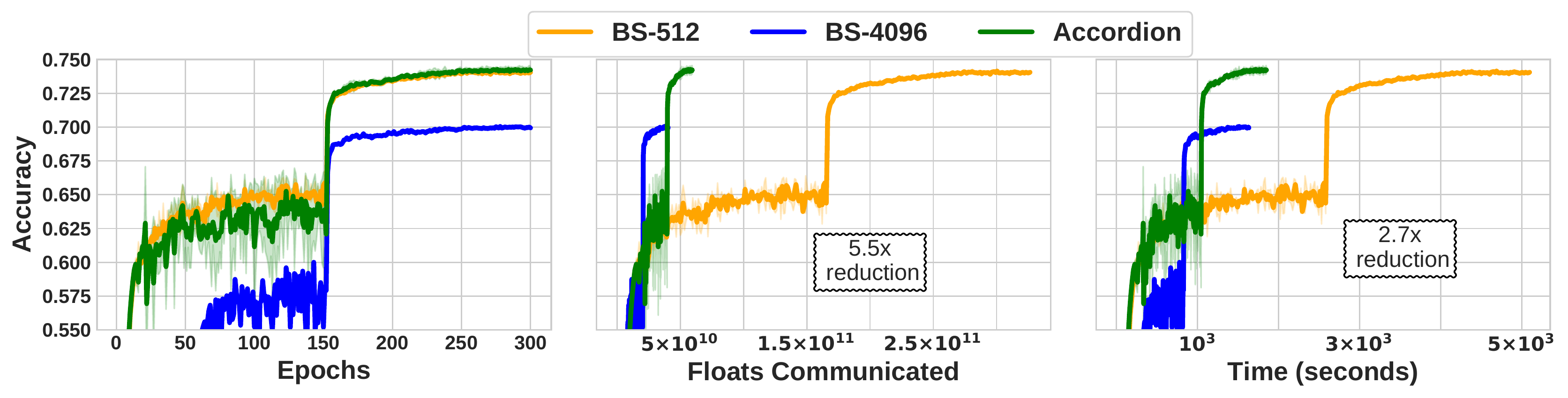}
        \caption{\small{{ResNet-18 trained using Batch Size=512,Batch Size=4096 and \adasp}}}
    \end{subfigure}
     \begin{subfigure}[b]{\textwidth}
        \includegraphics[width=\textwidth]{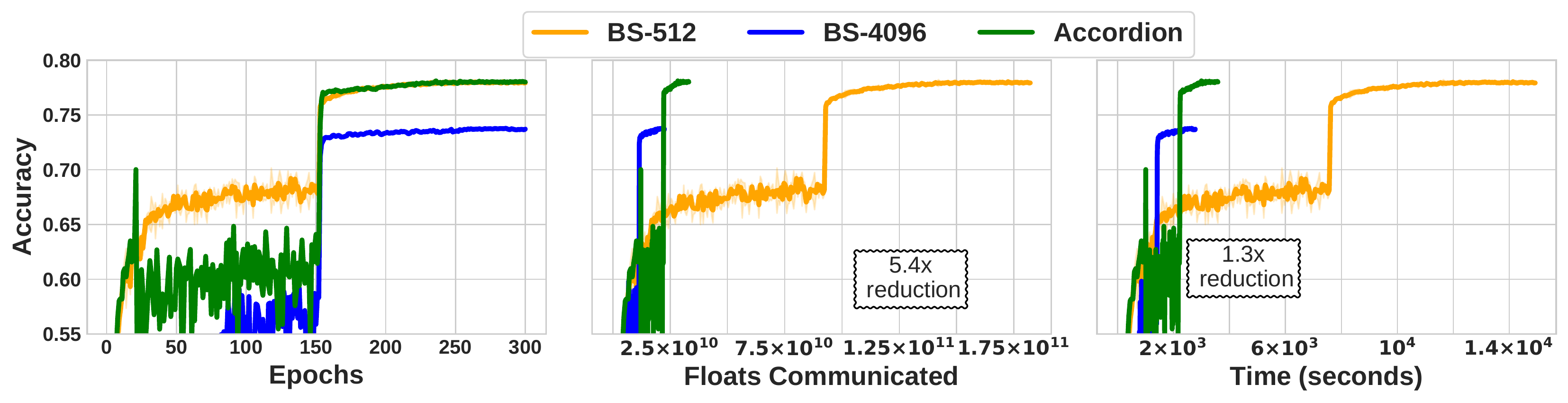}
        \caption{\small{{GoogLeNet trained using Batch Size=512,Batch Size=4096 and \adasp}}}
    \end{subfigure}
    \begin{subfigure}[b]{\textwidth}
        \includegraphics[width=\textwidth]{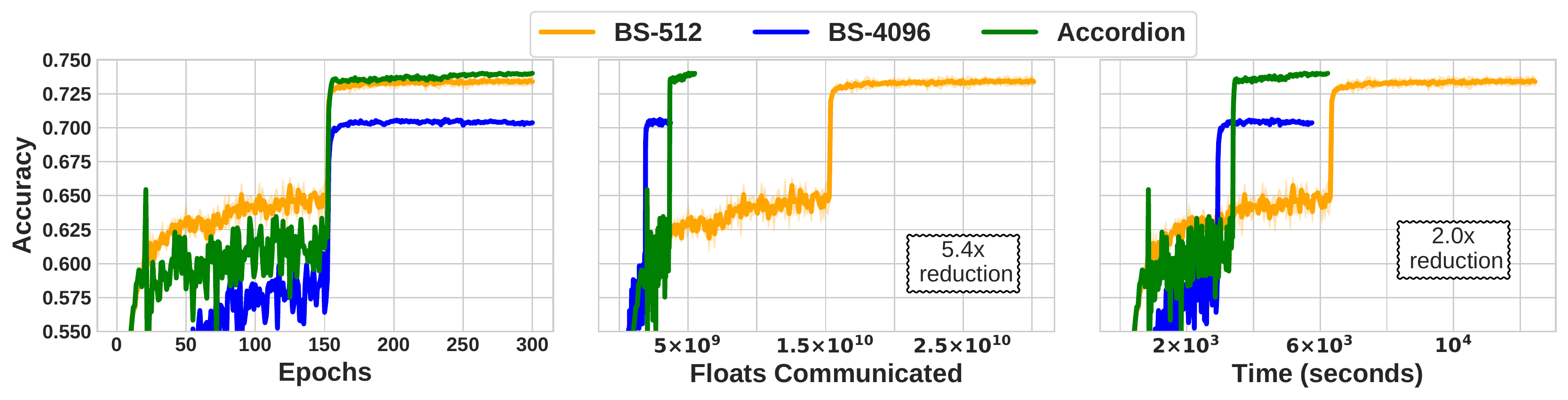}
        \caption{\small{{DenseNet trained using Batch Size=512,Batch Size=4096 and \adasp }}}
    \end{subfigure}
    \end{center}
\end{figure*}

\FloatBarrier
\section{Compression Ratio Selection of Adasparse}
In~\Cref{fig:k_chosen_1,fig:k_chosen_2,fig:k_chosen_3} we show the compression ratio chosen by \adasp\ for different layers when training ResNet-18 on \cifarh\ with \powersgd\ as a gradient compressor. The layer numbers are associated with how \pytorch\ indexes layers in the model. The missing layers numbers are 1 dimensional vectors which can not be compressed by \powersgd. 
\begin{figure}[!htb]
    \includegraphics[width=0.9\linewidth]{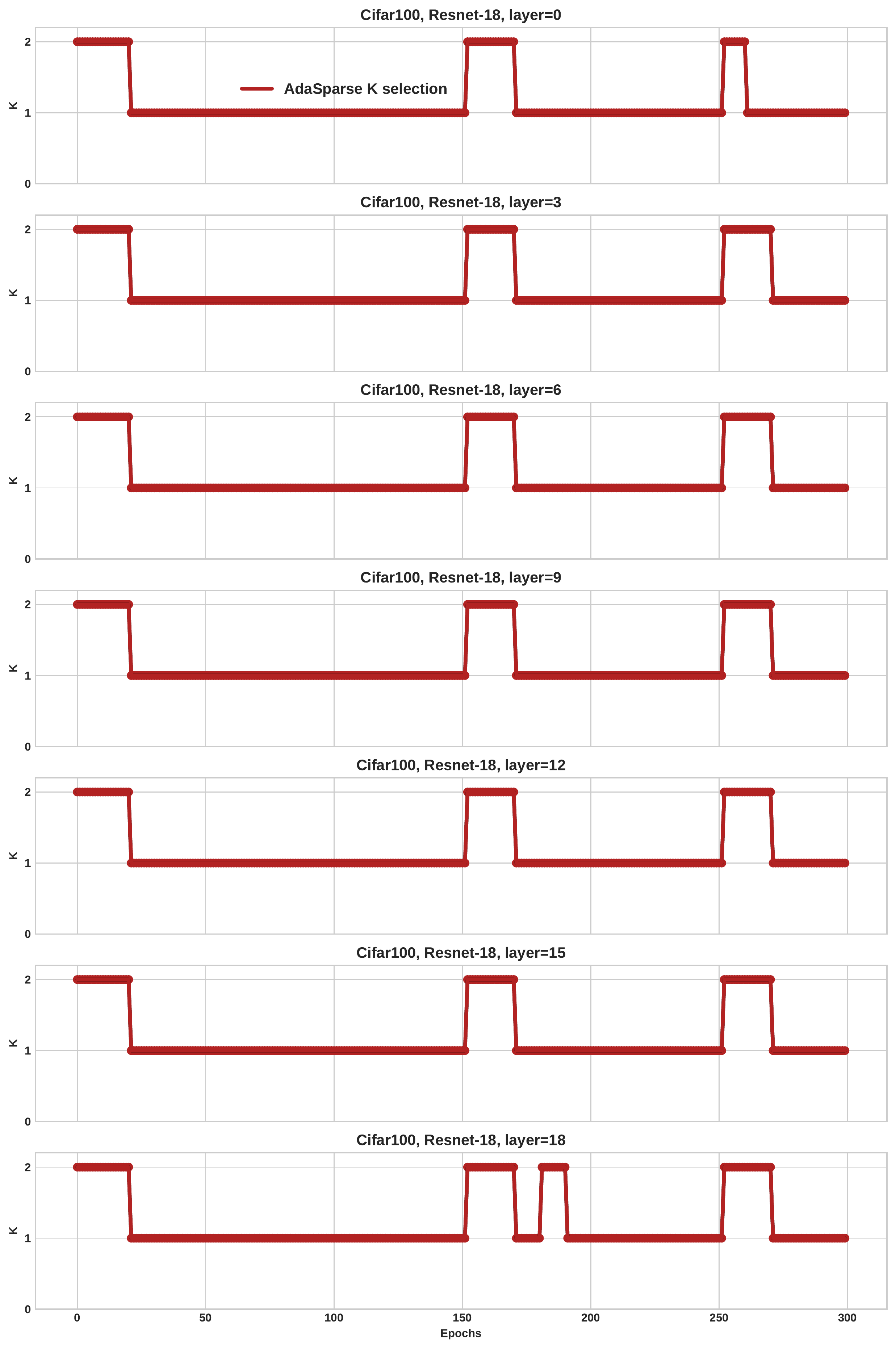}
    \caption{\small{\textbf{Rank selected in Different Regions by \adasp\ when used with \powersgd}}}
    \label{fig:k_chosen_1}
\end{figure}

\begin{figure}[ht]
    \includegraphics[width=0.9\linewidth]{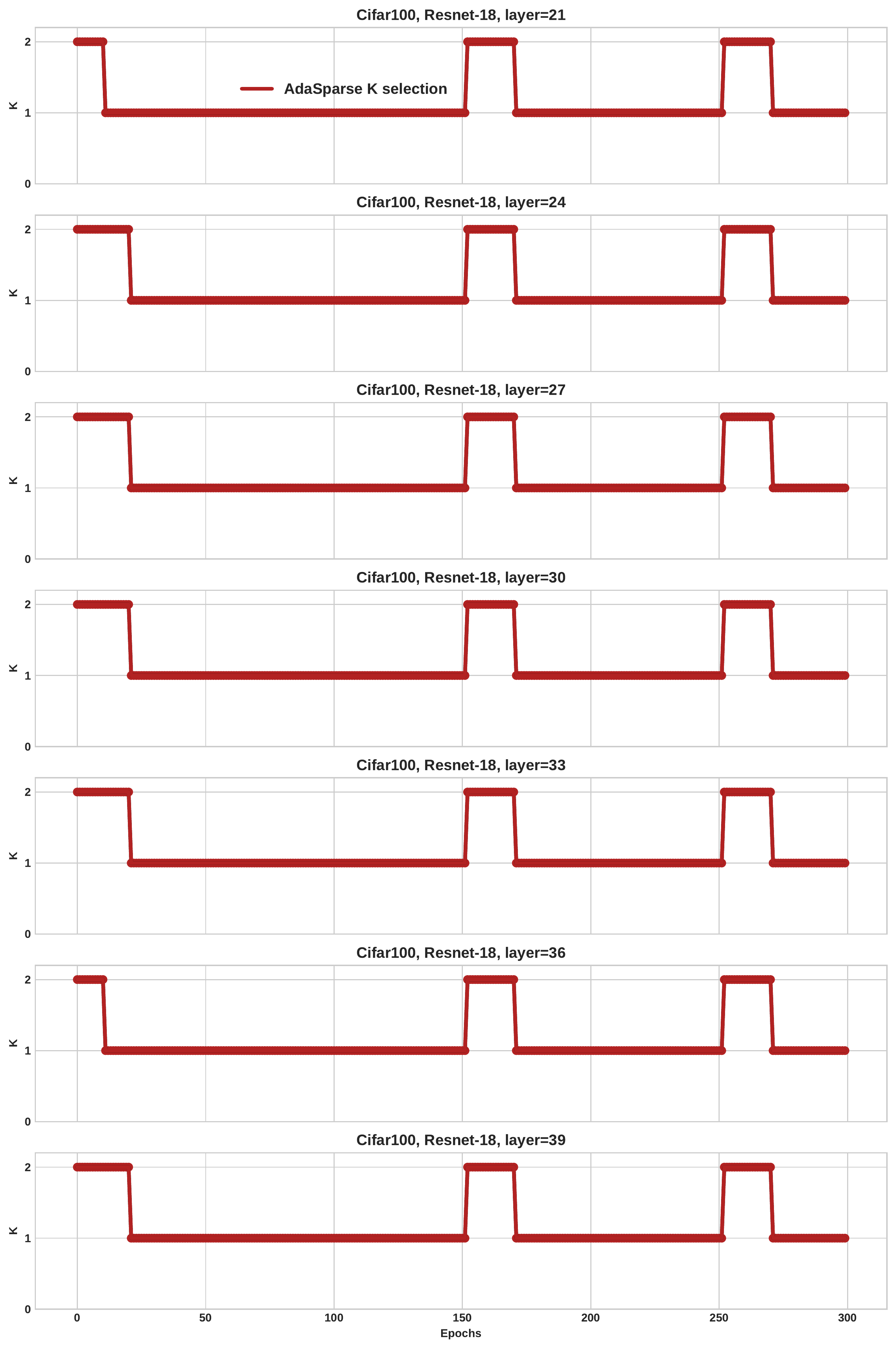}
    \caption{\small{\textbf{Rank selected in Different Regions by \adasp\ when used with \powersgd}}}
    \label{fig:k_chosen_2}
\end{figure}

\begin{figure}[ht]
    \includegraphics[width=0.9\linewidth]{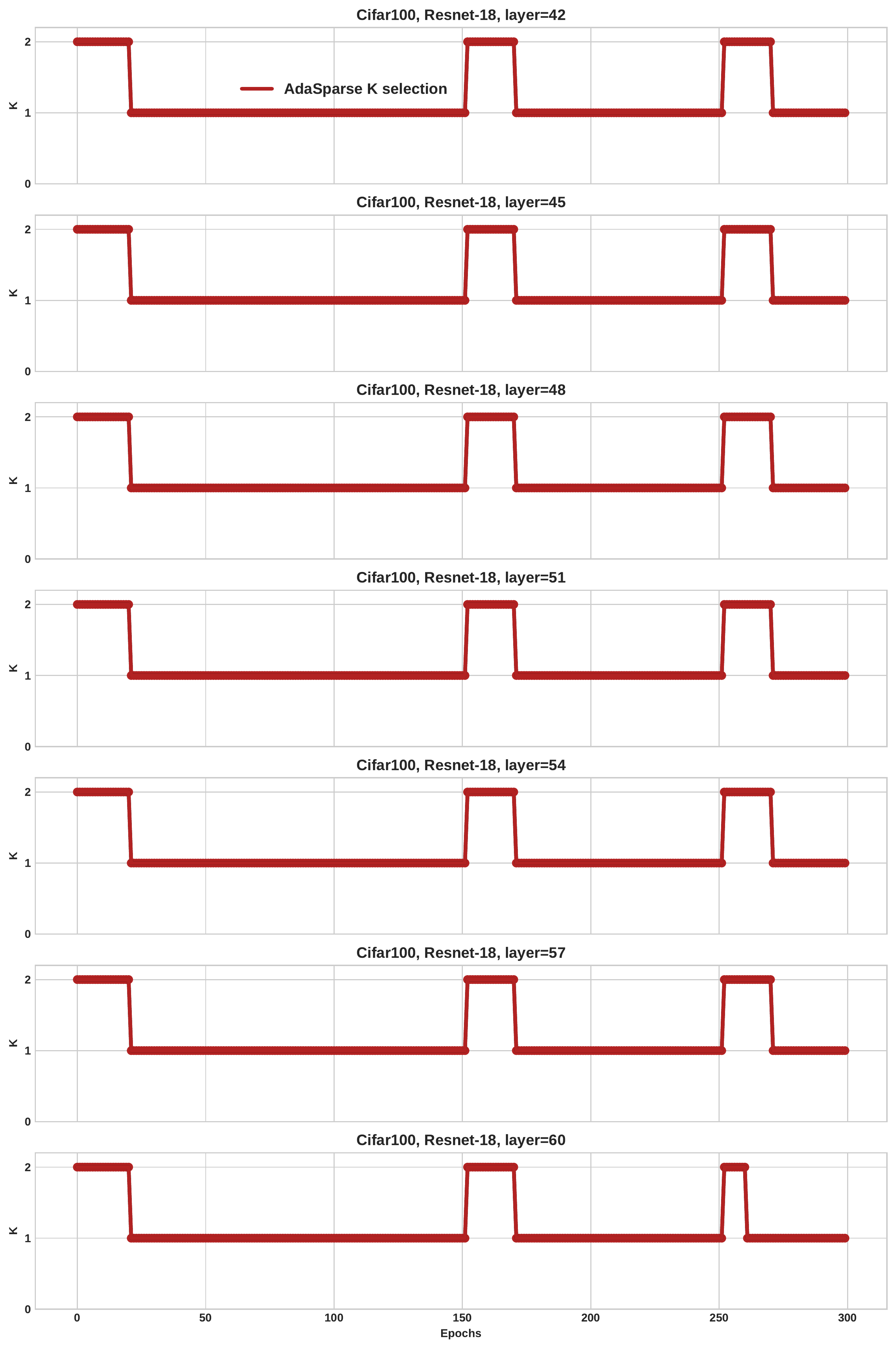}
    \caption{\small{\textbf{Rank selected in Different Regions by \adasp\ when used with \powersgd}}}
    \label{fig:k_chosen_3}
\end{figure}

\FloatBarrier
\section{Model Descriptions}
We use standard implementation for all the models. The model implementations for \cifart\ are borrowed from~\cite{pytorch-cifar10} and for \cifarh\ are borrowed from~\cite{pytorch-cifar100}. Here we present the total number of parameters in each of the model used. 
\begin{table}[!htb]
\begin{minipage}{0.48\textwidth}
\centering
\caption{\small{Total parameters when training \cifart}}
\label{tab:num_param_cifart}
\begin{tabular}{@{}l|l@{}}
\toprule
\textbf{Network} & \textbf{Total Parameters} \\ \midrule
ResNet-18        & 11173962                  \\
VGG-19bn         & 20565834                  \\
SeNet            & 11260354                  \\
WideResNet       & 36489290                  \\
DenseNet         & 1000618                   \\
GoogLeNet        & 6166250                   \\\bottomrule
\end{tabular}
\end{minipage}\quad
\begin{minipage}{0.48\textwidth}
\centering
\caption{\small{Total parameters when training \cifarh}}
\label{tab:num_param_cifarh}
\begin{tabular}{@{}l|l@{}}
\toprule
\textbf{Network} & \textbf{Total Parameters} \\ \midrule
ResNet-18        & 11220132                  \\
VGG-19bn         & 39327652                  \\
SeNet            & 11436256                  \\ 
WideResNet       & 36546980                  \\ 
DenseNet         & 1035268                   \\
GoogLeNet        & 6258500                   \\\bottomrule
\end{tabular}
\end{minipage}
\end{table}

\begin{table}[!htb]
    \centering
    \caption{\small{Total parameters when training \wikitext}}
    \label{tab:num_param_wiki}
    \begin{tabular}{@{}l|l@{}}
    \toprule
    \textbf{Network} & \textbf{Total Parameters} \\ \midrule
    2 Layer LSTM        & 28949319                   \\\bottomrule
\end{tabular}
\end{table}


\end{document}